\documentclass{article}
\usepackage{authblk}

\usepackage[utf8]{inputenc} 
\usepackage[T1]{fontenc}    
\usepackage{hyperref}       
\usepackage{url}            
\usepackage{booktabs}       
\usepackage{amsfonts}       
\usepackage{nicefrac}       
\usepackage{microtype}      
\usepackage{fullpage}
\usepackage{epstopdf}
\usepackage{amsthm}
\usepackage{mathtools,amssymb,algorithm,caption,cases,bm,float,graphicx,url,color}
\usepackage{subcaption}
\usepackage{sidecap}
\newcommand{\bb}{\bm \beta}
\newcommand{\mmu}{\bm \mu}

\newtheorem{lem}{Lemma}
\newtheorem{theorem}{Theorem}

\newtheorem{rem}{Remark}
\newtheorem{cor}{Corollary}

\title{The Impact of Regularization on High-dimensional Logistic Regression}

\author{Fariborz Salehi}
\author{Ehsan Abbasi}
\author{Babak Hassibi}
\affil{Department of Electrical Engineering\\
California Institute of Technology\\
Pasadena, CA 91125}

\begin{document}

\maketitle

\begin{abstract}
Logistic regression is commonly used for modeling dichotomous outcomes. In the classical setting, where the number of observations is much larger than the number of parameters, properties of the maximum likelihood estimator in logistic regression are well understood. Recently, Sur and Candes~\cite{sur2018modern} have studied logistic regression in the high-dimensional regime, where the number of observations and parameters are comparable, and show, among other things, that the maximum likelihood estimator is biased. In the high-dimensional regime the underlying parameter vector is often structured (sparse, block-sparse, finite-alphabet, etc.) and so in this paper we study regularized logistic regression (RLR), where a convex regularizer that encourages the desired structure is added to the negative of the log-likelihood function. An advantage of RLR is that it allows parameter recovery even for instances where the (unconstrained) maximum likelihood estimate does not exist. We provide a precise analysis of the performance of RLR via the solution of a system of six nonlinear equations, through which any performance metric of interest (mean, mean-squared error, probability of support recovery, etc.) can be explicitly computed. Our results generalize those of Sur and Candes and we provide a detailed study for the cases of $\ell_2^2$-RLR and sparse ($\ell_1$-regularized) logistic regression. In both cases, we obtain explicit expressions for various performance metrics and can find the values of the regularizer parameter that optimizes the desired performance. The theory is validated by extensive numerical simulations across a range of parameter values and problem instances.
\end{abstract}
\section{Introduction} \label{sec:intro}
Logistic regression is the most commonly used statistical model for predicting dichotomous outcomes~\cite{hosmer2013applied}. It has been extensively  employed in many areas of engineering and applied sciences, such as in the medical~\cite{boyd1987evaluating,tu1996advantages} and social sciences~\cite{king2001logistic}. As an example, in medical studies logistic regression can be used to predict the risk of developing a certain disease (e.g. diabetes) based on a set of observed characteristics from the patient (age, gender, weight, etc.)\\
Linear regression is a very useful  tool for predicting a quantitive response. However, in many situations the response variable is qualitative (or categorical) and linear regression is no longer appropriate~\cite{james2013introduction}. This is mainly due to the fact that least-squares often succeeds under the assumption that the error components are independent with normal distribution. In categorical predictions, however,  the error components are neither inependent nor normally distributed~\cite{nelder1972generalized}.  \\
In logistic regression we model the probability that the label, $Y$, belongs to a certain category. When no prior knowledge is available regarding the structure of the parameters, maximum likelihood is often used for fitting the model. Maximum likelihood estimation (MLE) is a special case of maximum a posteriori estimation (MAP) that assumes a uniform prior distribution on the parameters. \\
In many applications in statistics, machine learning, signal processing, etc., the underlying parameter obeys some sort of \emph{structure} (sparse, group-sparse, low-rank, finite-alphabet, etc.). For instance, in modern applications where the number of features far exceeds the number of observations, one typically enforces the solution to contain only a few non-zero entries. To exploit such structural information,  inspired by the Lasso~\cite{tibshirani1996regression} algorithm for linear models, researchers have studied regularization methods for generalized linear models~\cite{shevade2003simple, friedman2010regularization}. From a statistical viewpoint, adding a regularization term provides a MAP estimate with a non-uniform prior distribution that has higher densities in the set of structured solutions.
\subsection{Prior work}
Classical results in logistic regression mainly concern the regime where the sample size, $n$, is overwhelmingly larger than the feature dimension $p$. It can be shown that in the limit of large samples when $p$ is fixed and $n\rightarrow \infty$, the maximum likelihood estimator provides an efficient estimate of the underlying parameter, i.e., an unbiased estimate with covariance matrix approaching the inverse of the the Fisher information~\cite{van2000asymptotic, lehmann2006testing}. However, in most modern applications in data science, the datasets often have a huge number of features, and therefore, the assumption $\frac{n}{p}\gg 1$ is not valid. Sur and Candes~\cite{candes2018phase, sur2018modern, sur2017likelihood} have recently studied the performance of the maximum likeliood estimator for logistic regression in the  regime where $n$ is proportional to $p$. Their findings challenge the conventional wisdom, as they have shown that in the linear asymptotic regime the maximum likelikehood estimate is not even unbiased. Their analysis provides the precise performance of the maximum likelihood estimator. \\
There have been many studies in the literature on the performance of regularized (penalized) logistic regression, where a regularizer is added to the negative log-likelihood function (a partial list includes~\cite{bunea2008honest, kakade2010learning, van2008high}). These studies often require the underlying parameter to be heavily structured. For example, if the parameters are sparse the sparsity is taken to be $o(p)$. Furthermore, they provide orderwise bounds on the performance but do not give a precise characterization of the quality of the resulting estimate. A major advantage of adding a regularization term is that it allows for recovery of the parameter vector even in regimes where the maximum likelihood estimate does not exist (due to an insufficient number of observations.) 
\subsection{Summary of contributions}
In this paper, we study regularized logistic regression (RLR) for parameter estimation in high-dimensional logistic models. Inspired by recent advances in the performance analysis of M-estimators for linear models~\cite{donoho2016high, el2013robust, thrampoulidis2018precise}, we precisely characterize the assymptotic performance of the RLR estimate. Our characterization is through a system of six nonlinear equations in six unknowns, through whose solution all locally-Lipschitz performance measures such as the mean, mean-squared error, probability of support recovery, etc., can be determined. In the special case when the regularization term is absent, our $6$ nonlinear equations reduce to the $3$ nonlinear equations reported in~\cite{sur2018modern}. When the regularizer is quadratic in parameters, the $6$ equations also simplifies to $3$. When the regularizer is the $\ell_1$ norm, which corresponds to the popular sparse logistic regression~\cite{koh2007interior, krishnapuram2005sparse}, our equations can be expressed in terms of $q$-functions, and quantities such as the probability of correct support recovery can be explicitly computed. Numerous numerical simulations validate the theoretical findings across a range of problem settings. To the extent of our knowledge, this is the first work that precisely characterizes the performance of the regularized logistic regression in high dimensions.\\
For our analysis, we utilize the recently developed {\bf{C}}onvex {\bf{G}}aussian {\bf{M}}in-max {\bf{T}}heorem ({\bf{CGMT}})~\cite{thrampoulidis2015regularized} which is a strengthened version of a classical Gaussian comparison inequality due to Gordon~\cite{gordon1985some}, and whose origins are in~\cite{stojnic2013framework}. Previously, the CGMT has been successfully applied to derive the precise performance in a number of applications such as regularized M-estimators~\cite{thrampoulidis2018precise}, analysis of the generalized lasso~\cite{miolane2018distribution, thrampoulidis2015regularized}, data detection in massive MIMO~\cite{abbasi2019performance, atitallah2017ber, thrampoulidis2019simple},  and PhaseMax in phase retrieval~\cite{dhifallah2018phase, salehi2018precise, salehi2018learning}.
\section{Preliminaries}
\subsection{Notations}
We gather here the basic notations that are used throughout
this paper. $\mathcal N (\mu, \sigma^2)$ denotes the normal distribution with mean $\mu$ and variance $\sigma^2$. $X \sim p_X$ implies that the random variable $X$ has a density $p_X$. $\overset{\text{P}}\longrightarrow$ and $\overset{\text{d}}\longrightarrow$ represent convergence in probability and in distribution, respectively. Lower letters are reserved for vectors and upper letters are for matrices. $\mathbf 1_d$, and $\mathbf I_d$ respectively denote the all-one vector and the identity matrix in dimension $d$. For a vector $\mathbf v$, $v_i$ denotes its $i^{\text{th}}$ entry, and $||\mathbf v||_p$ (for $p\geq 1$),  is its $\ell_p$ norm, where we  remove the subscript when $p=2$. A function $f:\mathbb R^p\rightarrow \mathbb R$ is called \emph{(invariantly) separable} if $f(\mathbf w) = \sum_{i=1}^p \tilde{f}(w_i)$ for all $\mathbf w\in R^p$, where $\tilde f(\cdot)$ is a real-valued function. 
 For a function $\Phi:\mathbb R^d\rightarrow \mathbb R$, the Moreau envelope associated with $\Phi(\cdot)$ is defined as,
\begin{equation}
\label{eq:Moreau}
M_{\Phi}(\mathbf v, t) = \min_{\mathbf x\in \mathbb R^d}~~\frac{1}{2t}||\mathbf v-\mathbf x||^2+\Phi(\mathbf x)~,
\end{equation}
and the proximal operator is the solution to this optimization, i.e.,
\begin{equation}
\text{Prox}_{t\Phi(\cdot)}(\mathbf v) = \arg\min_{\mathbf x\in \mathbb R^d}~~\frac{1}{2t}||\mathbf v-\mathbf x||^2+\Phi(\mathbf x)~.
\end{equation}
\subsection{Mathematical Setup}
Assume we have $n$ samples from a logistic model with parameter $\bb^{*}\in \mathbb R^p$. Let $\mathcal D=\{(\mathbf x_i, y_i)\}_{i=1}^{n}$ denote the set of samples (a.k.a. the training data), where for $i=1,2,\ldots,n$, $\mathbf x_i\in \mathbb R^p$ is the feature vector and the label $y_i\in\{0,1\}$ is a Bernouli random variable with,
\begin{equation}
\mathbb P[y_i = 1|\mathbf x_i] = \mathbb \rho'(\mathbf x_i^T \bb^{*})~,~~\text{for }~i=1,2,\ldots,n~,
\end{equation} 
 where $\rho'(t) := \frac{e^t}{1+e^t}$ is the standard logistic function. The goal is to compute an estimate for $\bb^{*}$ from the training data $\mathcal D$. The maximum likelihood estimator, $\hat\bb_{ML}$, is defined as,
\begin{equation}
\begin{aligned}
&&\hat\bb_{ML} = \arg\max_{\bb\in \mathbb R^p} ~\overset{n}{\underset{i=1}\prod}~\mathbb P_{\bb}(y_i|\mathbf x_i)&=\arg\max_{\bb\in \mathbb R^p} ~\overset{n}{\underset{i=1}\prod}~\frac{e^{y_i(\mathbf x_i^T\bb)}}{1+e^{\mathbf x_i^T\bb}}\\
&&&=\arg\min_{\bb\in \mathbb R^p} ~\overset{n}{\underset{i=1}\sum}~\rho(\mathbf x_i^T \bb) - y_i (\mathbf x_i^T\bb)~.
\end{aligned}
\end{equation}
Where $\rho(t) := \log(1+e^t)$ is the \emph{link function} which has the standard logistic function as its derivative. The last optimization is simply minimization over the negative log-likelihood. This is a convex optimization program as the log-likelihood is concave with respect to $\bb$. \\
As explained earlier in Section~\ref{sec:intro}, in many interesting settings the underlying parameter possesses cerain structure(s) (sparse, low-rank, finite-alphabet, etc.). In order to exploit this structure  we assume $f:\mathbb R^p \rightarrow \mathbb R$ is a \emph{convex} function that measures the (so-called) "complexity" of the structured solution. We fit this model by the regularized maximum (binomial) likelihood defined as follows,
\begin{equation}
\label{eq:regularizedopt}
\hat \bb = \arg\min_{\bb\in \mathbb R^p}~\frac{1}{n}\cdot\big[\overset{n}{\underset{i=1}\sum}~\rho(\mathbf x_i^T \bb) - y_i (\mathbf x_i^T\bb)\big] + \frac{\lambda}{p}~ f(\bb) ~.
\end{equation}
Here, $\lambda\in \mathbb R_{+}$ is the regularization parameter that must be tuned properly. In this paper, we study the linear asymptotic regime in which the problem dimensions $p,n$ grow to infinity at a  proportional rate, $\delta:=\frac{n}{p}>0$. Our main result characterizes the performance of $\hat \bb$ in terms of the ratio, $\delta$, and the signal strength, $\kappa = \frac{||\bb^{*}||}{\sqrt{p}}$ . For our analysis we assume that the regularizer $f(\cdot)$ is separable, $f(\mathbf w)=\sum_i\tilde f(w_i)$, and the data points are drawn independently from the Gaussian distribution, $\{ \mathbf x_i\}_{i=1}^{n}\overset{\text{i.i.d.}}\sim\mathcal N(\mathbf 0, \frac{1}{p}\mathbf I_p)$. We further assume that the entries of $\bb^{*}$ are drawn from a distribution $\Pi$. Our main result characterizes the performance of the resulting estimator through the solution of a system of six nonlinear equations with six unknowns.  In particular, we use the solution to  compute some common descriptive statistics of the estimate, such as the mean and the variance.   
\section{Main Results}
\label{sec:main}
In this section, we present the main result of the paper, that is the characterization of the asymptotic performance of regularized logistic regression (RLR). When the estimation performance is measured via a locally- Lipschitz function (e.g. mean-squared error), Theorem~\ref{thm:main} precisely predicts the asymptotic behavior of the error. The derived expression captures the role of the regularizer, $f(\cdot)$, and the particular distribution of $\bb^*$, through a set of scalars derived by solving a system of nonlinear equations. In Section~\ref{sec:nonlinsys} we present this system of nonlinear equations along with some insights on how to numerically compute its solution. After formally stating our result in Section~\ref{sec:main_thm}, we use that to predict the general behavior of $\hat \bb$. In particular, in Section~\ref{sec:bias_var} we compute its correlation with the true signal as well as its mean-squared error. 
\subsection{A nonlinear system of equations} 
\label{sec:nonlinsys}
As we will see in Theorem~\ref{thm:main}, given the signal strength $\kappa$, and the ratio $\delta$, the asymptotic performance of RLR is characterized by the solution to the following system of nonlinear equations with six unknowns $(\alpha, \sigma, \gamma, \theta, \tau, r)$.
\begin{equation}
\label{eq:nonlinsys}
\begin{cases}
\begin{aligned}
&&\kappa^2 \alpha &= ~\mathbb E\big[\beta~\text{Prox}_{\lambda\sigma\tau \tilde f(\cdot)}\big(\sigma\tau(\theta \beta+\frac{r}{\sqrt{\delta}}Z)\big)\big]~,\\
&&\gamma &= \frac{1}{r\sqrt{\delta}}~\mathbb E\big[Z~\text{Prox}_{\lambda\sigma\tau \tilde f(\cdot)}\big(\sigma\tau(\theta \beta+\frac{r}{\sqrt{\delta}}Z)\big)\big]~,\\
&&\kappa^2\alpha^2 + \sigma^2 &= ~ \mathbb E~\big[\text{Prox}_{\lambda\sigma\tau \tilde f(\cdot)}\big(\sigma\tau(\theta \beta+\frac{r}{\sqrt{\delta}}Z)\big)^2\big]~,\\
&&\gamma^2 &= \frac{2}{r^2}~\mathbb E\big[\rho'(-\kappa Z_1)\big(\kappa\alpha Z_1+\sigma Z_2 -\text{Prox}_{\gamma\rho(\cdot)}(\kappa\alpha Z_1+\sigma Z_2) \big)^2\big]~,\\
&&\theta\gamma&=-2~\mathbb E\big[ \rho''(-\kappa Z_1)\text{Prox}_{\gamma \rho(\cdot)}\big(\kappa \alpha Z_1+\sigma Z_2\big)\big]~,\\
&&1-\frac{\gamma}{\sigma\tau}&=~\mathbb E\big[\frac{2\rho'(-\kappa Z_1)}{1+\gamma \rho''\big(\text{Prox}_{\gamma\rho(\cdot)}(\kappa\alpha Z_1 + \sigma Z_2)\big)}\big]~.
\end{aligned}
\end{cases}
\end{equation}

Here $Z,Z_1,Z_2$ are standard normal variables, and $\beta\sim \Pi$, where $\Pi$ denotes the distribution on the entries of $\bb^{*}$. The following remarks provide some insights on solving the nonlinear system. 
\begin{rem}[Proximal Operators]\label{rem:Nonlinear1}
 It is worth noting that the equations in~\eqref{eq:nonlinsys} include the expectation of functionals of two proximal operators. The first three equations are in terms of  $\text{Prox}_{\tilde f(\cdot)}$, which can be computed explicitly for most widely used regularizers. For instance, in $\ell_1$-regularization, the proximal operator is the well-known shrinkage function defined as $\eta(x,t) := \frac{x}{|x|}(|x|-t)_+$. The remaining equations depend on computing the proximal operator of the link function $\rho(\cdot)$. For $x\in \mathbb R$, $\text{Prox}_{t\rho(\cdot)}(x)$ is the unique solution of $z+t\rho'(z) = x$.
\end{rem}
\begin{rem}[Numerical Evaluation]
\label{rem:Nonlinear2}
Define $\mathbf v := [\alpha, \sigma, \gamma, \theta, \tau, r]^T$ as the vector of unknonws. The nonlinear system~\eqref{eq:nonlinsys} can be reformulated as $\mathbf v = S(\mathbf v)$ for a properly defined $S:\mathbb R^6\rightarrow \mathbb R^6$. We have empirically observed in our numerical simulations that a fixed-point iterative method, $\mathbf v_{t+1}=S(\mathbf v_t)$, converges to $\mathbf v^*$, such that $\mathbf v^* = S(\mathbf v^*)$.
\end{rem}
\subsection{Asymptotic performance of regularized logistic regression}
\label{sec:main_thm}
We are now able to present our main result. Theorem~\ref{thm:main} below describes the average behavior of the entries of $\hat \bb$, the solution of the  RLR. The derived expression is in terms of the solution of the nonlinear system~\eqref{eq:nonlinsys}, denoted by $(\bar \alpha, \bar \sigma, \bar \gamma, \bar \theta, \bar \tau, \bar r)$.  An {\underline{informal}} statement of our result is that as $n\rightarrow \infty$, the entries of $\hat \bb$ converge as follows,
\begin{equation}
    \hat\bb_j \overset{d}{\rightarrow} \Gamma(\bb_j^*,Z)~,~~\text{ for }~j=1,2,\ldots,p~, 
\end{equation}
where $Z$ is a standard normal random variable, and $\Gamma:\mathbb R^2\rightarrow \mathbb R$ is defined as,
\begin{equation}
\label{eq:Gamma_def}
    \Gamma(c,d) := \text{Prox}_{\lambda \bar\sigma \bar \tau\tilde f(\cdot)}\big(\bar\sigma\bar\tau(\bar \theta c+\frac{\bar r}{\sqrt{\delta}}d)\big)~. 
\end{equation}
In other words, the RLR solution has the same behavior as applying the proximal operator on the "perturbed signal", i.e., the true signal added with a Gaussian noise.
\begin{theorem}
\label{thm:main}
Consider the optimization program~\eqref{eq:regularizedopt}, where for $i=1,2,\ldots, n$, $\mathbf x_i$ has the multivariate Gaussian distribution $\mathcal N(0,\frac{1}{p}\mathbf I_p)$, and $y_i =  Ber(\mathbf x_i^T\bb^{*})$, and the entries of $\bb^{*}$ are drawn independently from a distribution $\Pi$. Assume the parameters $\delta$, $\kappa$, and $\lambda$ are such that the nonlinear system~\eqref{eq:nonlinsys} has a unique solution $(\bar\alpha, \bar \sigma, \bar \gamma, \bar \theta, \bar \tau, \bar r)$. Then, as $p\rightarrow \infty$, for any locally-Lipschitz\footnote{A function $\Phi:\mathbb R^d\rightarrow \mathbb R$ is said to be \emph{locally-Lipschitz} if,
\begin{equation*}
    \forall M>0,~\exists L_M\geq 0,~\text{such that}~~\forall \mathbf x, \mathbf y \in \big[-M,+M\big]^d~:~~~|\Phi(\mathbf x) -  \Phi(\mathbf y)|\leq L_{M}||\mathbf x-\mathbf y||~. 
\end{equation*}} function $\Psi:\mathbb R\times \mathbb R\rightarrow \mathbb R$ , we have,
\begin{equation}
    \frac{1}{p}\sum_{j=1}^p \Psi(\hat \bb_j, {\bb}_j^*) \overset{\text{P}}\longrightarrow \mathbb E\big[\Psi\big(\Gamma(\beta,Z),\beta\big)\big]~,
\end{equation}
where $Z\sim \mathcal N(0,1)$, $\beta\sim \Pi$ is independent of $Z$, and the function $\Gamma(\cdot,\cdot)$ is defined in~\eqref{eq:Gamma_def}.
\end{theorem}
We defer the detailed proof to the Appendix. In short, to show this result we first represnt the optimization as a bilinear form $\mathbf u^T\mathbf X\mathbf v$, where $\mathbf X$ is the measurement matrix. Applying the CGMT to derive an equivalent optimization, we then simplify this optimization to obtain an unconstrained optimization with six scalar variables. The nonlinear system~\eqref{eq:nonlinsys} represents the first-order optimality condition of the resulting scalar optimization.\\
Before stating the consequences of this result, a few remarks are in order.
\begin{rem}[Assumptions]\label{rem:main1} The assumptions in Theorem~\ref{thm:main} are chosen in a conservative manner. In particular, we could relax the separability condition on $f(\cdot)$, to some milder condition in terms of asymptotic convergence of its proximal operator. Furthermore, one can relax the assumption on the entries of $\bb^*$ being i.i.d. to a weaker assumption on  the empirical distribution of its entries. However, for the applications of this paper, the theorem in its current form is adequate.
\end{rem}
\begin{rem}[Choosing $\Psi$] \label{rem:main2}
The performance measure in Theorem~\ref{thm:main} is computed in terms of evaluation of a locally-Lipschitz function, $\Psi(\cdot,\cdot)$ . As an example, $\Psi(u,v)=(u-v)^2$ can be used to compute the mean-squared error. Later on, we will appeal to this theorem with various choices of $\Psi$ to evaluate different performance measures on $\hat \bb$. 
\end{rem}
\subsection{Correlation and variance of the RLR estimate}
\label{sec:bias_var}
As the first application of Theorem~\ref{thm:main} we compute common descriptive statistics of the estimate $\hat \bb$.  In the following corollaries, we establish that the parametrs $\bar \alpha$, and $\bar \sigma$ in~\eqref{eq:nonlinsys} correspond to the correlation and the mean-squared error of the resulting estimate.
\begin{cor}
\label{cor:bias}
As $p\rightarrow \infty$, $\frac{1}{||\bb^*||^2}~\hat\bb^T\bb^* \overset{P}\longrightarrow \bar \alpha$~.
\end{cor}
\begin{proof}
Recall that $||\bb^*||^2 = p\kappa^2$. Applying Theorem~\ref{thm:main} with $\Psi(u,v) = uv$ gives,
\begin{equation}
    \frac{1}{||\bb^*||^2} \hat \bb^T \bb^* = \frac{1}{\kappa^2p} \sum_{j=1}^p {\hat\bb}_j\bb^*_j \overset{P}\longrightarrow  \frac{1}{\kappa^2}\mathbb E\big[\beta~\text{Prox}_{\lambda \bar\sigma \bar \tau\tilde f(\cdot)}\big(\bar\sigma\bar\tau(\bar \theta \beta+\frac{\bar r}{\sqrt{\delta}}Z)\big)\big] = \bar \alpha~,
\end{equation}
where the last equality is derived from the first equation in the nonlinear system~\eqref{eq:nonlinsys}, along with the fact that $(\bar \alpha, \bar \sigma, \bar \gamma, \bar \theta, \bar \tau, \bar r)$ is a solution to this system.
\end{proof}
Corollary~\ref{cor:bias} states that upon centering $\hat \bb$ around $\bar\alpha \bb^*$, it becomes decorrelated from $\bb^*$. Therefore, we define a new estimate $\tilde \bb := \frac{\hat \bb}{\bar \alpha}$ and compute its mean-squared error in the following corollary.
\begin{cor}
\label{cor:var}
As $p\rightarrow \infty$, $\frac{1}{p}||\tilde \bb - \bb^* ||^2\overset{P}\longrightarrow  \frac{{\bar \sigma}^2}{{\bar \alpha}^2}$~.
\end{cor}
\begin{proof}
We appeal to Theorem~\ref{thm:main} with $\Psi(u,v) = (u-\bar \alpha v)^2$,
\begin{equation}
    \frac{1}{p}||\tilde \bb - \bb^* ||^2 = \frac{1}{{\bar \alpha}^2}\big(\frac{1}{p}||\hat \bb - \bar \alpha\bb^* ||^2\big)\overset{P}\longrightarrow \frac{1}{{\bar \alpha}^2}\mathbb E\big[\big(\text{Prox}_{\lambda \bar\sigma \bar \tau\tilde f(\cdot)}\big(\bar\sigma\bar\tau(\bar \theta \beta+\frac{\bar r}{\sqrt{\delta}}Z)\big)-\bar\alpha \beta \big)^2\big] = \frac{{\bar \sigma}^2}{{\bar \alpha}^2}~,
\end{equation}
where the last equality is derived from the third equation in the nonlinear system~\eqref{eq:nonlinsys} together with the result of Corollary~\ref{cor:bias}.
\end{proof}
In the next two sections, we investigate other properties of the estimate $\hat\bb$ under $\ell_1$ and $\ell_2$ regularization.
\section{RLR with $\ell_2^2$-regularization}
The $\ell_2$ norm regularization is commonly used in machine learning applications to stabilize the model. Adding this regularization  would simply shrink all the parameters toward the origin and hence  decrease the variance of the resulting model. Here, we provide a precise performance analysis of the RLR with $\ell_2^2$-regularization, i.e.,
\begin{equation}
\label{eq:l2regularizedopt}
\hat \bb = \arg\min_{\bb\in \mathbb R^p}~\frac{1}{n}\cdot\big[\overset{n}{\underset{i=1}\sum}~\rho(\mathbf x_i^T \bb) - y_i (\mathbf x_i^T\bb)\big] + \frac{\lambda}{2p}\sum_{i=1}^p \bb_i^2 ~.
\end{equation}
To analyze~\eqref{eq:l2regularizedopt}, we use the result of Theorem~\ref{thm:main}. It can be shown that in the nonlinear system~\eqref{eq:nonlinsys}, $\bar \theta$, $\bar \tau$, $\bar r$ can be derived explicitely from solving the first three equations. This is due to the fact that the proximal operator of $\tilde f(\cdot) = \frac{1}{2}(\cdot)^2$ can be expressed in the following closed-form,
\begin{equation}
\label{eq:prox_l2}
   \text{Prox}_{t\tilde f(\cdot)}(x) = \arg\min_{y\in \mathbb R}~\frac{1}{2t}(y-x)^2+\frac{1}{2}y^2 = \frac{x}{1+t}~.
\end{equation}
This indicates that the proximal operator in this case is just a simple rescaling. Substituting~\eqref{eq:prox_l2} in the nonlinear system~\eqref{eq:nonlinsys}, we can rewrite the first three equations as follows,
\begin{equation}
    \begin{cases}
    \begin{aligned}
    &&\theta &=\frac{\alpha}{\gamma \delta}~,\\
    &&\tau &=\frac{\delta \gamma}{\sigma\big(1-\lambda \delta \gamma\big)}~,\\
    && r &= \frac{\sigma}{ \gamma \sqrt{\delta}}~.
    \end{aligned}
    \end{cases}
\end{equation}
Therefore we can state the following Theorem for $\ell_2^2$-regularization:
\begin{figure}[t]
	\centering
	\begin{subfigure}[b]{0.45\textwidth}
		\centering
		\includegraphics[width=\textwidth]{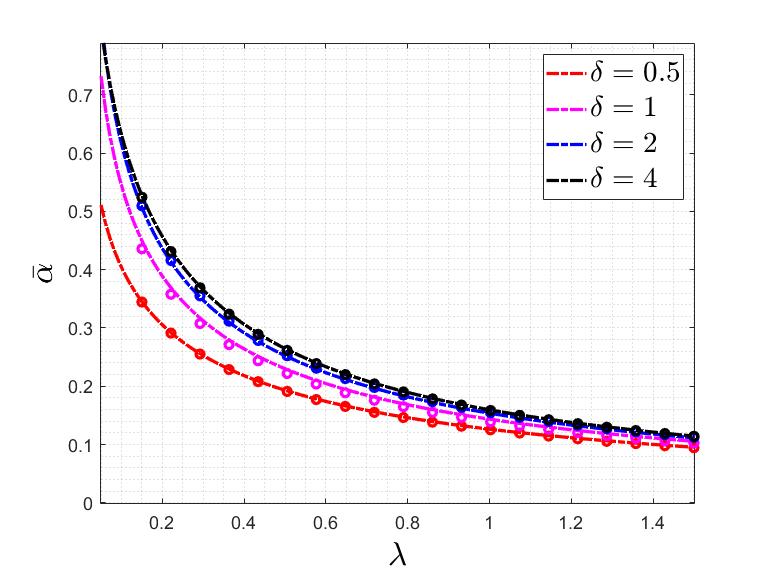}
		\caption{\small{ ~}}
		\label{fig:fig1a}
	\end{subfigure}
	\begin{subfigure}[b]{0.45\textwidth}
		\centering
		\includegraphics[width=\textwidth]{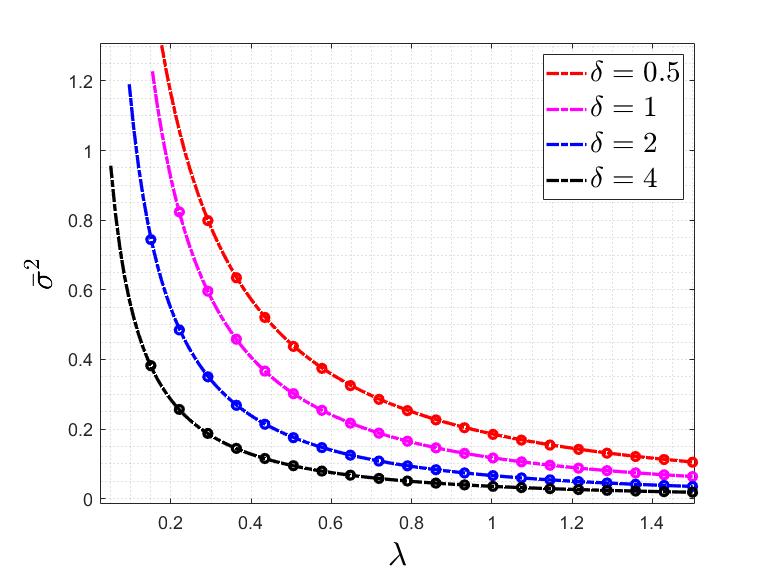}
		\caption{\small{ ~}}
		\label{fig:fig1b}
	\end{subfigure}
	\begin{subfigure}[b]{0.45\textwidth}
		\centering
		\includegraphics[width=\textwidth]{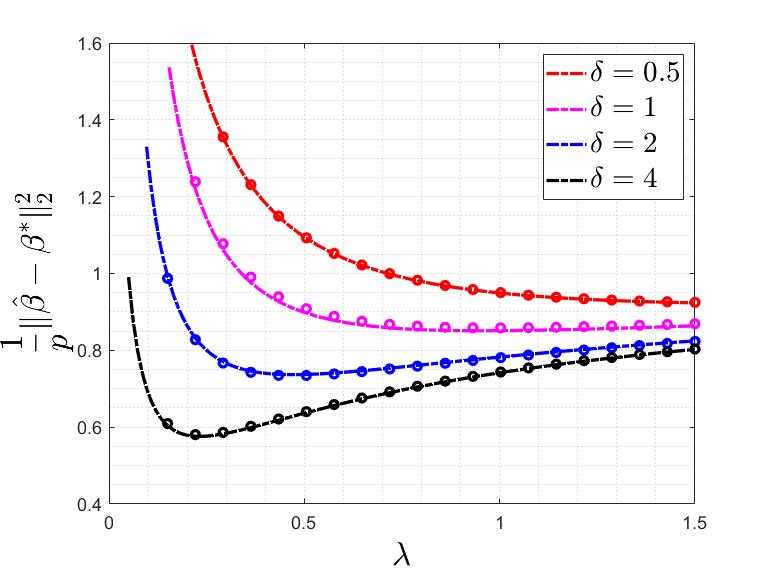}
		\caption{\small{ ~}}
		\label{fig:fig1c}
	\end{subfigure}
	\caption{\small{The performance of the regularized logistic regression under $\ell_2^2$ penalty  (a) the correlation factor $\bar \alpha$ (b) the variance ${\bar \sigma}^2$, and (c) the mean-squared error $\frac{1}{p}{||\hat \bb-\bb^*||}^2$. The dashed lines depict the theoretical result derived from Theorem~\ref{thm:l2_reg}}, and the dots are the result of empirical simulations. The empirical results is the average over $100$ independent trials with $p=250$ and $\kappa=1$~.}
\label{fig:fig11}
\end{figure}
\begin{theorem}\label{thm:l2_reg}
Consider the optimization~\eqref{eq:l2regularizedopt} with parameters $\kappa$, $\delta$, and $\gamma$, and the same assumptions as in Theorem~\ref{thm:main}. As $p\rightarrow \infty$, for any locally-Lipschitz function $\Psi(\cdot,\cdot)$, the following convergence holds,
\begin{equation}
    \frac{1}{p}\sum_{j=1}^p \Psi(\hat \bb_j-\bar \alpha {\bb}_j^*, {\bb}_j^*) \overset{\text{P}}\longrightarrow \mathbb E\big[\Psi\big(\bar \sigma Z,\beta\big)\big]~,
\end{equation}
where $Z$ is standard norma, $\beta\sim \Pi$, and $\bar \alpha$,$\bar \sigma$ are the unique solution to the following nonlinear system of equations,
\begin{equation}
\label{eq:nonlin_l2}
    \begin{cases}
    \begin{aligned}
    &&\frac{\sigma^2}{2\delta} &= ~\mathbb E\big[\rho'(-\kappa Z_1)\big(\kappa\alpha Z_1+\sigma Z_2 -\text{Prox}_{\gamma\rho(\cdot)}(\kappa\alpha Z_1+\sigma Z_2) \big)^2\big]~,\\
    &&-\frac{\alpha}{2\delta}&=~\mathbb E\big[ \rho''(-\kappa Z_1)\text{Prox}_{\gamma \rho(\cdot)}\big(\kappa \alpha Z_1+\sigma Z_2\big)\big]~,\\
    &&1-\frac{1}{\delta}+\lambda\gamma&=~\mathbb E\big[\frac{2\rho'(-\kappa Z_1)}{1+\gamma \rho''\big(\text{Prox}_{\gamma\rho(\cdot)}(\kappa\alpha Z_1 + \sigma Z_2)\big)}\big]~.
    \end{aligned}
    \end{cases}
\end{equation}
\end{theorem}
The proof is deferred to the Appendix. Theorem~\ref{thm:l2_reg} states that upon centering the estimate $\hat \bb$, it becomes decorrelated from $\bb^*$ and the distribution of the entries approach a  zero-mean Gaussian distribution with variance ${\bar\sigma}^2$. \\
Figure~\ref{fig:fig11} depicts the performance of the regularized estimate for different values of $\lambda$. As observed in the figure, increasing the value of $\lambda$ reduces the correlation factor $\bar \alpha$ (Figure~\ref{fig:fig1a})  and the variance ${\bar \sigma}^2$ (Figure~\ref{fig:fig1b}). Figure~\ref{fig:fig1c} shows the mean-squared-error of the estimate as a function of $\lambda$ . It indicates that for different values of $\delta$ there exist an optimal value $\lambda_{\text{opt}}$ that achieves the minimum mean-squared error.  
\subsection{Unstructured case}
When $\lambda=0$ in~\eqref{eq:l2regularizedopt}, we obtain the optimization with no regularization, i.e., the maximum likelihood estimate. When we set $\lambda$ to zero in~\eqref{eq:nonlin_l2}, Theorem~\ref{thm:l2_reg} gives the same result as Sur and Candes reported in~\cite{sur2018modern}. In their analysis, they have also provided an interesting interpretation of $\bar \gamma$ in terms of the likelihood ratio statistics. Studying the likelihood ratio test is beyond the scope of this paper.   
\section{Sparse Logistic Regression}\label{sec:sparse}
In this section we study the performance of our estimate when the regularizer is the $\ell_1$ norm. In modern machine learning applications the number of features, $p$, is often overwhelmingly large. Therefore, to avoid overfitting one typically needs to perform feature selection, that is, to exclude irrelevent variables from the regression model~\cite{james2013introduction}.  Adding an $\ell_1$ penalty to the loss function is the most popular approach for feature selection.\\
As a natural consequence of the result of Theorem~\ref{thm:main}, we study the performance of RLR with $\ell_1$ regularizer (referred to as "sparse LR") and evaluate its success in recovery of the sparse signals. In Section~\ref{sec:sparse_s1}, we extend our general analysis to the case of sparse LR. In other words, we will precisely analyze the performance of the solution of the following optimization,
\begin{equation}
\label{eq:sparseLRopt}
\hat \bb = \arg\min_{\bb\in \mathbb R^p}~\frac{1}{n}\cdot\big[\overset{n}{\underset{i=1}\sum}~\rho(\mathbf x_i^T \bb) - y_i (\mathbf x_i^T\bb)\big] + \frac{\lambda}{p}||\bb||_1 ~.
\end{equation}
In Section~\ref{sec:sparse_s1}, we explicitly describe the expectations in the nonlinear system~\eqref{eq:nonlinsys} using two  $q$-functions\footnote{The $q$-function is the tail distribution of the standard normal r.v. defined as, $Q(t) := \int_{t}^{\infty}\frac{e^{-x^2/2}}{\sqrt{2\pi}}dx$~.}. In Section~\ref{sec:sparse_s2}, we analyze the support recovery in the resulting estimate and show that the two $q$-functions represent the probability of on and off support recovery. 
\subsection{Convergence behavior of sparse LR}\label{sec:sparse_s1}
For our analysis in this section, we assume each entry $\bb^*_i$, for $i = 1, \ldots , p$, is sampled i.i.d. from a distribution,
\begin{equation}
\label{eq:sparse_distribution}
    \Pi(\beta) = (1-s)\cdot\delta_0(\beta) + s\cdot \big(\frac{\phi(\frac{\beta}{\frac{\kappa}{\sqrt s}})}{\frac{\kappa}{\sqrt s}}\big) ,
\end{equation}
where $s\in(0,1)$ is the \emph{sparsity factor}, $\phi(t):=\frac{e^{-t^2/2}}{\sqrt{2\pi}}$ is the density of the standard normal distribution, and $\delta_0(\cdot)$ is the Dirac delta function. In other words, entries of $\bb^*$ are zero with probability $1-s$, and the non-zero entries have a Gaussian distribution with appropriately defined variance. Although our analysis can be extended further, here we only present the result for a Gaussian distribution on the non-zero entries. The proximal operator of $\tilde f(\cdot) =|\cdot|$ is the soft-thresholding operator defined as, $\eta(x,t) = \frac{x}{|x|}(x-t)_+~.$
Therefore, we are able to explicitly compute the expectations with respect to $\tilde f(\cdot)$ in the nonlinear system~\eqref{eq:nonlinsys}. To streamline the representation, we define the following two proxies,
\begin{equation}
    \label{eq:proxy:sparse}
    t_1 = \frac{\lambda}{\sqrt{\frac{r^2}{\delta}+\frac{\theta^2\kappa^2}{s}}}~,~~
    t_2=\frac{\lambda}{\frac{r}{\sqrt \delta}}~.
\end{equation}
In the next section, we provide an interpretation for $t_1$ and $t_2$. In particular, we will show that $Q(\bar{t_1})$, and $Q(\bar{t_2})$ are related to the probabilities of on and off support recovery. We can rewrite the first three equations in~\eqref{eq:nonlinsys} as follows,
\begin{equation}
\label{eq:nonlinsys_sparse}
    \begin{cases}
    \begin{aligned}
    &&\frac{\alpha}{2\sigma \tau} &=~\theta\cdot Q(t_1)~,\\
    &&\frac{\delta\gamma}{2\sigma \tau} &= s\cdot Q\big(t_1\big) + (1-s)\cdot Q\big(t_2\big)~,\\
    &&\frac{\kappa^2\alpha^2+\sigma^2}{2\sigma^2\tau^2}&=\frac{\delta\gamma\lambda^2}{2\sigma\tau}+\frac{\gamma r^2}{2\sigma\tau}+\kappa^2\theta^2\cdot Q\big(t_1\big)-\lambda^2(s\cdot\frac{\phi(t_1)}{t1} + (1-s)\cdot\frac{\phi(t_2)}{t_2})~.
    \end{aligned}
    \end{cases}
\end{equation}
Appending the three equations in~\eqref{eq:nonlinsys_sparse} to the last three equations in~\eqref{eq:nonlinsys} gives the nonlinear system  for sparse LR. Upon solving these equations, we can use the result of Theorem~\ref{thm:main} to compute various performance measure on the estimate $\hat \bb$.
\begin{figure*}[t]
	\centering
	\begin{subfigure}[b]{0.45\textwidth}
		\centering
		\includegraphics[width=\textwidth]{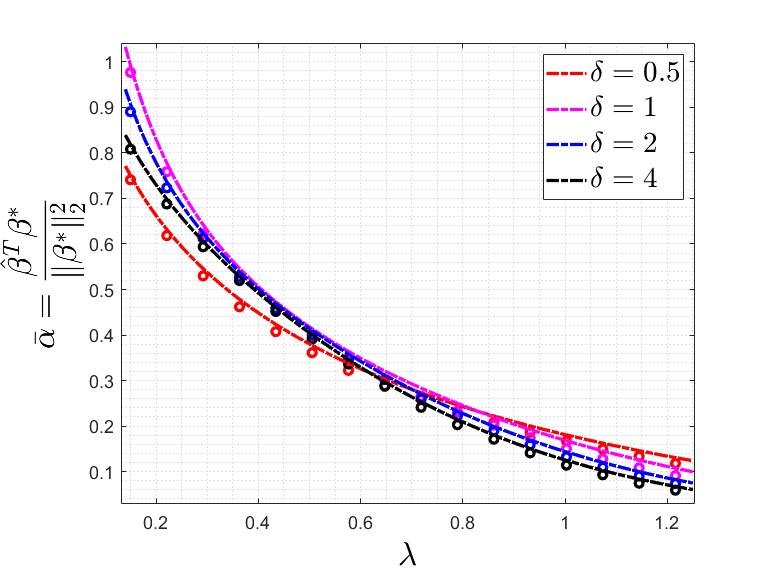}
		\caption{\small{ ~}}
		\label{fig:fig2a}
	\end{subfigure}
	\begin{subfigure}[b]{0.45\textwidth}
		\centering
		\includegraphics[width=\textwidth]{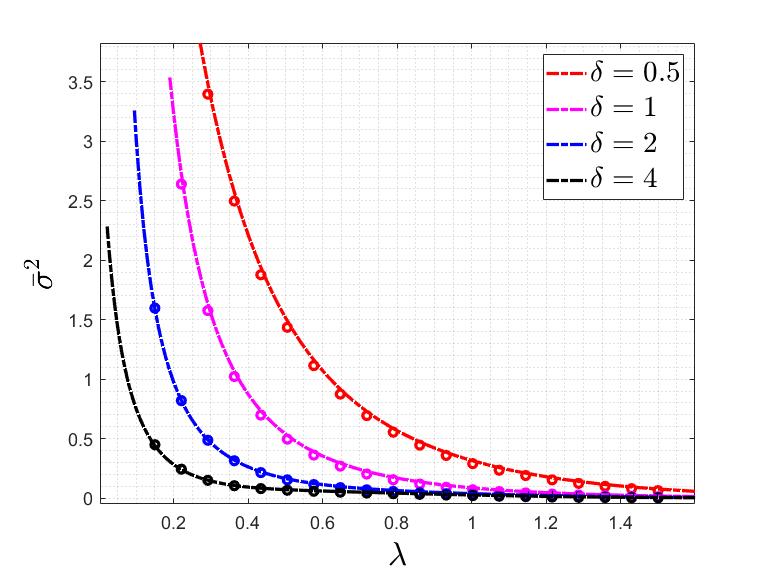}
		\caption{\small{ ~}}
		\label{fig:fig2b}
	\end{subfigure}
	\begin{subfigure}[b]{0.45\textwidth}
		\centering
		\includegraphics[width=\textwidth]{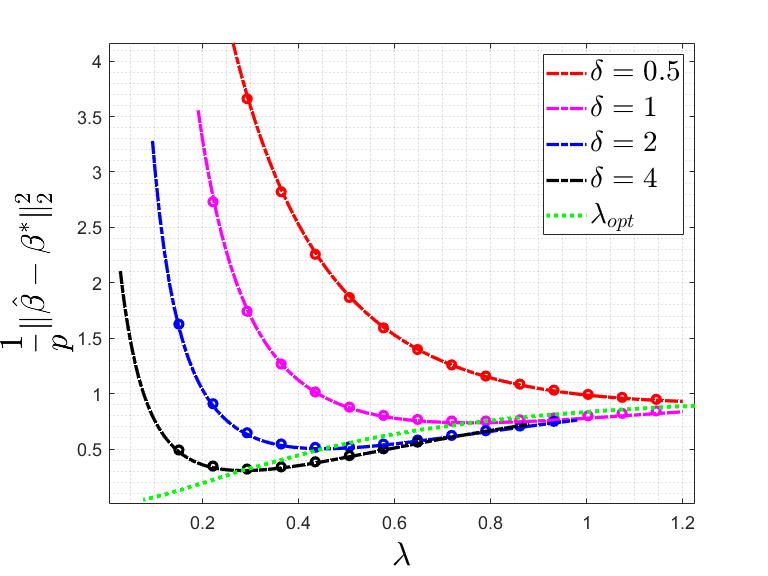}
		\caption{\small{ ~}}
		\label{fig:fig2c}
	\end{subfigure}
	\caption{\small{The performance of the regularized logistic regression under $\ell_1$ penalty  (a) the correlation factor $\bar \alpha$ (b) the variance ${\bar \sigma}^2$, and (c) the mean-squared error $\frac{1}{p}{||\hat \bb-\bb^*||}^2$. The dashed lines are the theoretical result derived from Theorem~\ref{thm:main}, and the dots are the result of empirical simulations. For the numerical simulations, the result is the average over $100$ independent trials with $p=250$ and $\kappa=1$~. }}
	\label{fig:fig2}
\end{figure*}
Figure~\ref{fig:fig2} shows the performance of our estimate as a function of $\lambda$. It can be seen that the bound derived from our theoretical result matches the empirical simulations. Also, it can be inferrred from Figure~\ref{fig:fig2c} that the optimal value of $\lambda$ ($\lambda_{\text{opt}}$ that achieves the minimum mean-squared error) is a decreasing function of $\delta$.
\subsection{Support recovery}\label{sec:sparse_s2}
In this section, we study the support recovery in sparse LR.  As mentioned earlier, sparse LR is often used when the underlying paramter has few non-zero entries. We define the support of $\bb^*$ as $\Omega := \{j|1\leq j \leq p, \bb^*_j\neq 0 \}$. Here, we would like to compute the probability of success in recovery of the support of $\bb^*$.\\
Let $\hat \bb$ denote the solution of the optimization~\eqref{eq:sparseLRopt}. We fix the value $\epsilon>0$ as a  hard-threshold based on which we decide whether an entry is on the support or not. In other words, we form the following set as our estimate of the support given $\hat \bb$,
\begin{equation}
    \hat \Omega = \{j| 1\leq j\leq p, |{\hat \beta}_j|>\epsilon\}
\end{equation}
In order to evaluate the success in support recovery, we define the following two error measures,
\begin{equation}
    E_1(\epsilon) = \text{Prob}\{j\in \hat\Omega|j\not \in \Omega\}~~,~~E_2(\epsilon) = \text{Prob}\{j\not \in \hat \Omega|j \in \Omega\}~.
\end{equation}
In our estimation, $E_1$ represents the probability of false alarm, and $E_2$ is the probability of misdetection of an entry of the support. The following lemma indicates the asymptotic behavior of both errors as $\epsilon$ approcahes zero~. 
\begin{lem}[Support Recovery]\label{lem:sup_rec}
Let $\hat \bb$ be the solution to the optimization~\eqref{eq:sparseLRopt}, and the entries of $\bb^*$ have distribution $\Pi$ defined in~\eqref{eq:sparse_distribution}. Assume $\lambda$ is chosen such that the nonlinear system~\eqref{eq:nonlinsys} has a unique solution $(\bar\alpha, \bar \sigma, \bar \gamma, \bar \theta, \bar \tau, \bar r)$. As $p\rightarrow \infty$ we have,
\begin{equation}
\begin{aligned}
    &&&~~~~\lim_{\epsilon\downarrow 0}E_1(\epsilon)\overset{p}\longrightarrow ~2~Q\big(\bar{t}_1\big)~~\text{where, }~~\bar{t}_1=\frac{\lambda}{\frac{\bar r}{\sqrt{\delta}}},~~\text{and,}\\
    &&&\lim_{\epsilon\downarrow 0}E_2(\epsilon)\overset{p}\longrightarrow 1-~2~Q\big({\bar t}_2\big)~~\text{where,}~~{\bar t}_2 = \frac{\lambda}{\sqrt{\frac{\bar r^2}{\delta}+\frac{\bar \theta^2\kappa^2}{s}}}~.
\end{aligned}
\end{equation}
\end{lem}
\begin{figure*}[htp]
	\centering
	\begin{subfigure}[b]{0.47\textwidth}
		\centering
		\includegraphics[width=\textwidth]{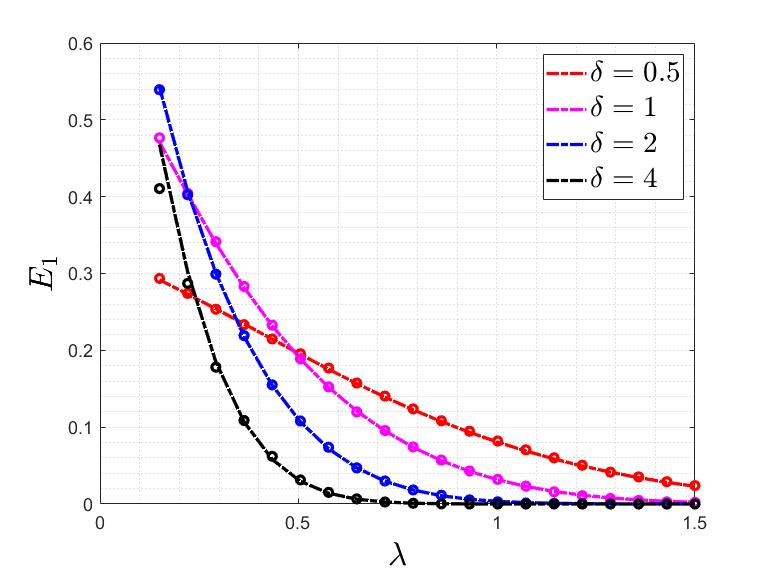}
		\caption{\small{ ~}}
		\label{fig:fig3a}
	\end{subfigure}
	\begin{subfigure}[b]{0.47\textwidth}
		\centering
		\includegraphics[width=\textwidth]{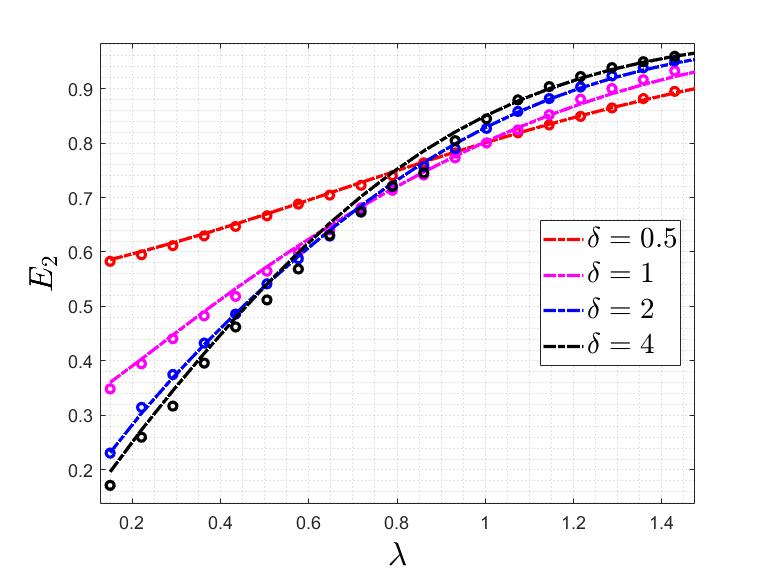}
		\caption{\small{ ~}}
		\label{fig:fig3b}
	\end{subfigure}
	\caption{\small{The support recovery in the regularized logistic regression with $\ell_1$ penalty for (a) $E_1$: the probability of false detection, (b) $E_2$: the probability of missing an entry of the support. The dashed lines are the theoretical results derived from Lemma~\ref{lem:sup_rec}, and the dots are the result of empirical simulations. For the numerical simulations, the result is the average over $100$ independent trials with $p=250$ and $\kappa=1$ and $\epsilon=0.001$~. }}\label{fig:fig3}
\end{figure*}

\section{Conclusion and Future Directions}
In this paper, we analyzed the performance of the regularized logistic regression (RLR), which is often used for parameter estimation in binary classification. We considered the setting where the underlying parameter has certain structure (e.g. sparse, group-sparse, low-rank, etc.) that can be enforced via a convex penalty function $f(\cdot)$. We precisely characterized the performance of the regularized maximum likelihood estimator via the solution to a nonlinear system of equations. Our main results can be used to measure the performance of RLR for a general  convex penalty function $f(\cdot)$. In particular, we apply our findings to two important special cases, i.e.,  $\ell_2^2$-RLR and $\ell_1$-RLR. When the regularizer is quadratic in parameters, we have shown that the nonlinear system can be simplified to three equations. When the regularization parameter, $\lambda$, is set to zero, which corresponds to the maximum likelihood estimator, we simply derived the results reported by Sur and Candes~\cite{sur2018modern}. For sparse logistic regression, we established that the nonlinear system can be represented using two $q$-functions. We further show that these two $q$-functions represent the probability of the support recovery.\\ 
For our analysis, we assumed the datapoints are drawn independently from a gaussian distribution and utilized the CGMT framework. An interesting future work is to extend our analysis to non-gaussian distributions.  To this end, we can exploit the techniques that have been used to establish the universality law (see ~\cite{oymak2017universality, panahi2017universal} and the references therein). As mentioned earlier in Section~\ref{sec:intro}, an advantage of RLR is that it allows parameter recovery even for instances where the (unconstrained) maximum likelihood estimate does not exist. Therefore, another interesting future direction is to analyze the conditions on $\lambda$ (as a function of $\delta$ and $\kappa$) that guarantees the existence of the solution to the RLR optimization~\eqref{eq:regularizedopt}. In the unstructured setting, this has been studied in a recent work by Candes and Sur~\cite{candes2018phase}.

\newpage
\bibliography{library}
\bibliographystyle{plain}
\newpage
\appendix
\section*{Appendix}
\label{sec:appendix}
\section{Convex Gaussian Min-max Theorem (CGMT)}
 Our analysis is based on the convex gaussian min-max theorem (CGMT). Here, we formally state this theorem. The CGMT associates with a Primary Optimization (PO) problem an Auxiliary Optimization (AO) problem from which we can investigate various properties of the primary optimization, such as the phase transition. In particular, the (PO) and the (AO) problems are defined respectively as follows:
\begin{subequations}\label{eq:POAO}
\begin{align}
\label{eq:PO_gen}
\Phi(\mathbf G)&:= \min_{\mathbf w\in\mathcal S_{\mathbf w}}~\max_{\mathbf u\in \mathcal S_{\mathbf u}}~ \mathbf u^T\mathbf G\mathbf w + \psi(\mathbf u, \mathbf w),\\
\label{eq:AO_gen}
\phi(\mathbf g,\mathbf h)&:= \min_{\mathbf w\in\mathcal S_{\mathbf w}}~\max_{\mathbf u\in \mathcal S_{\mathbf u}}~ ||\mathbf w||\mathbf g^T\mathbf u - ||\mathbf u||\mathbf h^T\mathbf w + \psi(\mathbf u, \mathbf w),
\end{align}
\end{subequations}
where $\mathbf G\in\mathbb R^{m\times n}, \mathbf g\in\mathbb R^m, \mathbf h\in\mathbb R^n$, $\mathcal S_{\mathbf w}\subset\mathbb R^n,\mathcal S_{\mathbf u}\subset\mathbb R^m$ and $\psi:\mathbb R^n\times\mathbb R^m\rightarrow\mathbb R$. Denote by $\mathbf w_\Phi:=\mathbf w_\Phi(\mathbf G)$ and $\mathbf w_\phi:=\mathbf w_\phi(\mathbf g,\mathbf h)$ any optimal minimizers in \eqref{eq:PO_gen} and \eqref{eq:AO_gen}, respectively.

\begin{theorem}[CGMT]\cite{thrampoulidis2016recovering}
In~\eqref{eq:POAO}, let $\mathcal S_\mathbf w$, $\mathcal S_{\mathbf u}$, be convex and compact sets, and assume $\psi(\cdot,\cdot)$ is convex-concave on $\mathcal S_{\mathbf w}\times\mathcal S_{\mathbf u}$. Also assume that $\mathbf G,~\mathbf g,$ and $\mathbf h$ all have entries i.i.d. standard normal. The following statements are true,
\begin{enumerate}
    \item for all $\mu\in \mathbb R$, and $t>0$,
    \begin{equation}
        \mathbb P(|\Phi(\mathbf G)-\mu|>t)\leq 2\mathbb P(|\phi(\mathbf g, \mathbf h)-\mu|\geq t)~.
    \end{equation}
    \item Let $\mathcal S$ be an arbitrary open subset of $\mathcal S_{\mathbf w}$ and $\mathcal S^c:=\mathcal S_{\mathbf w}/\mathcal S$. Denote $\Phi_{\mathcal S^c}(\mathbf G)$ and $\phi_{\mathcal S^c}(\mathbf g, \mathbf h)$ be the optimal costs of the optimizations in~\eqref{eq:PO_gen}, and ~\eqref{eq:PO_gen}, respectively, when the minimization over $\mathbf w$ is now constrained over $\mathbf w \in \mathcal S^c$. If there exists constants $\bar \phi$, ${\bar \phi}_{\mathcal S^c}$, and $\eta>0$ such that,
    \begin{itemize}
        \item ${\bar \phi}_{\mathcal S^c}\geq \bar \phi +3\eta$~,\\
        \item $\phi(\mathbf g, \mathbf h) < \bar \phi +\eta$, with probability at least $1-p$~,\\
        \item $\phi_{\mathcal S^c}(\mathbf g, \mathbf h)>{\bar \phi}_{\mathcal S^c}-\eta$, with probability at least $1-p$~,
    \end{itemize}
    then,
    $\mathbb P(\mathbf w_{\Phi}(\mathbf G)\in \mathcal S)\geq 1-4p$~.
\end{enumerate}
The probabilities are taken with respect to the randomness in $\mathbf G$, $\mathbf g$, and $\mathbf h$. 
\label{thm:CGMT}
\end{theorem}
We also use the following corollary that is true in the asymptotic regime,
\begin{cor}[Asymptotic CGMT]~\cite{thrampoulidis2016recovering} using the same notations and assumptions as in Theorem~\ref{thm:CGMT}, suppose there exists constants $\bar \phi<{\bar \phi}_{\mathcal S^c}$ such that $\phi(\mathbf g, \mathbf h)\overset{p}\longrightarrow \bar \phi$, and $\phi_{\mathcal S^c}(\mathbf g, \mathbf h)\longrightarrow {\bar \phi}_{\mathcal S^c}$. Then, 
\begin{equation}
    \lim_{n\rightarrow \infty}\mathbb P(\mathbf w_{\Phi}(\mathbf G)\in \mathcal S)=1~. 
\end{equation}
\label{cor:CGMT}
\end{cor}
We refer the interested reader to~\cite{thrampoulidis2016recovering, thrampoulidis2015regularized, thrampoulidis2018precise} for furder reading on the subject, its premises and applications.
\section{Useful Mathematical Tools}\label{sec:math_tools}
We gathered here some useful lemmas that are used in the proof of our main results. The first lemma provides the partial derivatives of the Moreau envelope function.
\begin{lem}
\label{lem:der_moreau}
Let $\Phi:\mathbb R^d\rightarrow \mathbb R$ be a convex function. For $\mathbf v \in \mathbb R^d$ and $t\in \mathbb R_{+}$, the \emph{Moreau envelope} function is defined as,
\begin{equation}
M_{\Phi(\cdot)}(\mathbf v,t) = \min_{\mathbf x \in \mathbb R^d} \Phi(\mathbf x) + \frac{1}{2t}||\mathbf x-\mathbf v||^2~,
\end{equation}
and the \emph{proximal operator} is the solution to this optimization, i.e., 
\begin{equation}
\text{Prox}_{t\Phi(\cdot)}(\mathbf v) = \arg\min_{\mathbf x \in \mathbb R^d} \Phi(\mathbf x) + \frac{1}{2t}||\mathbf x-\mathbf v||^2~.  
\end{equation}
The derivative of the Moreau envelope function can be computed as follows,
\begin{equation}
    \frac{\partial M_{\Phi(\cdot)}}{\partial \mathbf v} = \frac{1}{t}(\mathbf v -\text{Prox}_{t\Phi(\cdot)}(\mathbf v) )~~,~~~~~\frac{\partial M_{\Phi(\cdot)}}{\partial t} = -\frac{1}{2t^2}(\mathbf v -\text{Prox}_{t\Phi(\cdot)}(\mathbf v) )^2~.
\end{equation}
\end{lem}
We refer the interested reader to~\cite{jourani2014differential} for the proof as well as a detailed study of the properties of the Moreau envelope.  

The next two lemmas present some properties of the proximal operator for the function $\rho(z) = \log(1+ e^z)$.
\begin{lem}
\label{lem:prox}
Let $\rho(z)=\log(1+e^z)$, then the following identity holds,
\begin{equation}
\text{Prox}_{t\rho}(x+t) = -\text{Prox}_{t\rho}(-x)~.
\end{equation}
\end{lem}
\begin{proof}
Since the function $\rho(\cdot)$ is differentiable the proximal operator satisfies the following equation,
\begin{equation}
\label{eq:prox_eq}
\frac{1}{t}(\text{Prox}_{t\rho(\cdot)}(x)-x) + \rho'(\text{Prox}_{t\rho(\cdot)}(x)) = 0~.
\end{equation}
Next we use the fact that  $\rho'(-z)= 1-\rho'(z)$ for $z\in \mathbb R$, to rewrite the equation as follows,
\begin{equation}
\frac{1}{t}\big(-\text{Prox}_{t\rho(\cdot)}(-x) - (x+t)\big) + \rho'(-\text{Prox}_{t\rho(\cdot)}(-x))=0~,
\end{equation}
which gives the desired identity.
\end{proof}

\begin{lem}
\label{lem:prox_2}
The derivative of the proximal operator of the function $\rho(\cdot)$ can be computed as follows,
\begin{equation}
\label{eq:prox_deriv}
    \frac{d}{dx}\text{Prox}_{t\rho(\cdot)}(x) = \frac{1}{1+t\rho^{''}\big(\text{Prox}_{t\rho(\cdot)}(x)\big)}~.
\end{equation}
\end{lem}
\begin{proof}
Taking derivative with respect to $x$ of~\eqref{eq:prox_eq},
\begin{equation}
    \frac{1}{t}(\frac{d}{dx}\text{Prox}_{t\rho(\cdot)}(x) - 1) + \frac{d}{dx}\text{Prox}_{t\rho(\cdot)}(x) \times \rho^{''}\big(\text{Prox}_{t\rho(\cdot)}(x)\big)  =0~,
\end{equation}
which can be written as in~\eqref{eq:prox_deriv}.
\end{proof}
\begin{lem}[Stein's lemma]~\cite{vershynin2018high} For a function $f:\mathbb R \rightarrow \mathbb R$, we have 
$\mathbb E_Z[Zf(Z)] = \mathbb E_Z[f'(Z)]$~.
\label{lem:Stein}
\end{lem}

\begin{lem}\label{lem:separable}
Let $f:\mathbb R^d\rightarrow \mathbb R$ be an invariantly separable function such that $f(\mathbf x) = \sum_{i=1}^d\tilde f(x_i)$ for all $\mathbf x\in \mathbb R^d$, where $\tilde f$ is a real-valued function. Then, we have:
\begin{equation}
    \label{eq:prox_separable}
    M_{f(\cdot)}(\mathbf v, t) = \sum_{i=1}^d{M_{\tilde f(\cdot)}(v_i, t)}~~,~~~\text{and}~~~~\text{Prox}_{tf(\cdot)}(\mathbf v)=\begin{bmatrix}
    \text{Prox}_{t\tilde f(\cdot)}(v_1)\\
    \text{Prox}_{t\tilde f(\cdot)}(v_2)\\
    \vdots\\
    \text{Prox}_{t\tilde f(\cdot)}(v_d)
    \end{bmatrix}~.
\end{equation}
\end{lem}
\begin{proof}
We can write,
\begin{equation}
\begin{aligned}
&&M_{f(\cdot)}(\mathbf v,t) = \min_{\mathbf x \in \mathbb R^d} f(\mathbf x) + \frac{1}{2t}||\mathbf x-\mathbf v||^2 &= \min_{\mathbf x \in \mathbb R^d} \sum_{i=1}^d \tilde f(x_i) + \frac{(x_i-v_i)^2}{2t}~,\\
&&&=\sum_{i=1}^d \min_{x_i} \tilde f(x_i) + \frac{(x_i-v_i)^2}{2t}~,\\
&&&=\sum_{i=1}^d M_{\tilde f(\cdot)}(v_i, t)~.
\end{aligned}
\end{equation}
\end{proof}

\section{Proof of Theorem~\ref{thm:main}}\label{sec:proof}
We present the proof of our main result that is a precise characterization on the performance of the optimization program~\eqref{eq:regularizedopt} in the limit where  $p,n\rightarrow \infty$ at a fixed ratio $\delta:=\frac{n}{p}$.  We assume the data points are drawn independently from Gaussian distribution, $\mathbf x_i \overset{\text{i.i.d.}}\sim \mathcal N(0,\frac{1}{p}\mathbf I_p)$. We first rewrite~\eqref{eq:regularizedopt} as follows,
\begin{equation}
\min_{\bb\in \mathbb R^p}~\frac{1}{n}\mathbf 1^T\rho(\frac{1}{\sqrt{p}}\mathbf H\bb) - \frac{1}{n\sqrt{p}}\mathbf y^T\mathbf H\bb +\frac{\lambda}{p} f(\bb)
\end{equation}
where the action of function $\rho(\cdot)$ on a vector is considered component-wise, $\mathbf y\in \mathbb R^n$ and $\mathbf H\in \mathbb R^{n\times p}$ are defined as follows,
\begin{equation}
\mathbf y = \begin{bmatrix}y_1\\y_2\\\vdots\\y_n\end{bmatrix}~,~\mathbf H = \sqrt{p}\cdot\begin{bmatrix}-\mathbf x_1^T-\\-\mathbf x_2^T-\\\vdots\\-\mathbf x_n^T-\end{bmatrix}~.
\end{equation}
Note that the matrix $\mathbf H$ is defined in such a way that its entries have i.i.d. standard normal distribution. \\
We use the CGMT framework for our analysis. The proof strategy consists of three main steps:
\begin{enumerate}
    \item Finding the auxiliary optimization: In order to apply the result of Theorem~\ref{thm:CGMT}, we need to rewrite the optimization as a bilinear form and find its associated auxiliary optimization. 
    \item Analyzing the auxiliary optimization: The goal of this step is to simplify the auxiliary optimization in such a way that its performance can be characterized via a scalar optimization.
    \item Finding the optimality condition on the scalar optimization: We investigate the solution to the resulting scalar optimization. Specifically, by writing the first-order optimality conditions, we will derive the nonlinear system of equations~\eqref{eq:nonlinsys}.
\end{enumerate}
We explain each of the three steps in more details in the following subsections.
\subsection{Finding the auxiliary optimization}\label{sec:find_AO}
In order to apply the CGMT, we need to have a min-max optimization. Introducing a new variable $\mathbf u$, we have the following optimization,
\begin{equation}
\label{eq:addconstraint}
\begin{aligned}
&&\min_{\bb\in \mathbb R^p,~\mathbf u\in \mathbb R^n}~&\frac{1}{n}\mathbf 1^T\rho(\mathbf u) - \frac{1}{n}\mathbf y^T\mathbf u +\frac{\lambda}{p} f(\bb)\\
&&& \text{s.t. }~\mathbf u = \frac{1}{\sqrt{p}}\mathbf H\mathbf \bb
\end{aligned}
\end{equation}
Next, we use the Lagrange multiplier $\mathbf v$ to rewrite~\eqref{eq:addconstraint} as a min-max optimization,
\begin{equation}
\label{eq:minmax1}
\min_{\bb\in \mathbb R^p,\mathbf u\in \mathbb R^n}\max_{\mathbf v \in \mathbb R^n }~\frac{1}{n}\mathbf 1^T\rho(\mathbf u) - \frac{1}{n}\mathbf y^T\mathbf u +\frac{\lambda}{p} f(\bb) + \frac{1}{n}\mathbf v^T(\mathbf u - \frac{1}{\sqrt{p}}\mathbf H \bb)~.
\end{equation}
Since $\mathbf y$ depends on $\mathbf H$ we can not directly apply CGMT to the bilinear form $\mathbf v^T\mathbf H \bb$. To solve this issue, we first introduce, $\mathbf P := \frac{1}{{||\bb^{*}||}_2^2}\bb^{*}{\bb^{*}}^T$, and $\mathbf P^{\perp} := \mathbf I_{p} - \mathbf P$,  the projection matrices on the direction of $\bb^{*}$ and its orthogonal complement, respectively. We use these projections to decompose the matrix $\mathbf H$ as, $\mathbf H = \mathbf H_1 + \mathbf H_2$, with $\mathbf H_1 := \mathbf H\times\mathbf P$, and $\mathbf H_2 := \mathbf H\times\mathbf P^{\perp}$.  Rewriting~\eqref{eq:minmax1} with the decomposition of  $\mathbf H$ would give,
\begin{equation}
\label{eq:primary}
\min_{\bb\in \mathbb R^p,\mathbf u\in \mathbb R^n}\max_{\mathbf v \in \mathbb R^n }~\frac{1}{n}\mathbf 1^T\rho(\mathbf u) - \frac{1}{n}\mathbf y^T\mathbf u +\frac{\lambda}{p} f(\bb) + \frac{1}{n}\mathbf v^T(\mathbf u - \frac{1}{\sqrt{p}}\mathbf H_1 \bb) - \frac{1}{n\sqrt{p}}\mathbf v^T\mathbf H_2 \bb~.
\end{equation}
It is worth noting that after performing this decomposition, the label vector  ($\mathbf y$) would be independent of $\mathbf H_2$ since,
\begin{equation}
\mathbf y = Ber\big(\rho'(\frac{1}{\sqrt p} \mathbf H \bb^{*})\big) = Ber\big(\rho'(\frac{1}{\sqrt p} \mathbf H \mathbf P\bb^{*})\big) = Ber\big(\rho'(\frac{1}{\sqrt p} \mathbf H_1 \bb^{*})\big)~,
\end{equation}
where we used $\mathbf P\bb^{*} = \bb^{*}$. Exploiting this fact, one can check that all the additive terms in the objective function of~\eqref{eq:primary} except the last one are independent of $\mathbf H_2$. Also, the objective function is convex with respect to $\bb$ and $\mathbf u$ and concave with respect to $\mathbf v$. In order to apply the CGMT framework, we only need an extra condition which is restricting the feasible sets of $\bb, \mathbf u$, and $\mathbf v$ to be compact and convex. We can introduce some artificial convex and bounded  sets $\mathcal S_\mathbf u$, $S_{\mathbf v}$, and $\mathcal S_{\bb}$, and perform the optimization over these sets. Note that these sets can be chosen large enough such that they do not affect the optimization itself. For simplicity, in our arguments here we ignore the condition on the compactness of the fesible sets and apply the CGMT whenever our feasible sets are convex.

 The optimization program~\eqref{eq:primary} is suitable to be analyzed via the CGMT as the conditions are all satisfied. Having identified~\eqref{eq:primary} as the (PO) in our optimization, it is straightforward to write its corresponding (AO) as in ~\eqref{eq:POAO}. Therefore, the Auxiliary Optimization (AO) can be written as follows,
\begin{align}
\label{eq:aux}
\min_{\bb\in \mathbb R^p,\mathbf u\in \mathbb R^n}\max_{\mathbf v \in \mathbb R^n }~\frac{1}{n}\mathbf 1^T\rho(\mathbf u) - \frac{1}{n}\mathbf y^T\mathbf u +\frac{\lambda}{p} f(\bb) + \frac{1}{n}\mathbf v^T(\mathbf u - \frac{1}{\sqrt{p}}\mathbf H_1 \bb)\nonumber\\
- \frac{1}{n\sqrt{p}}(\mathbf v^T\mathbf h||\mathbf P^{\perp}\bb||+||\mathbf v||\mathbf g^T\mathbf P^{\perp}\bb)~,
\end{align}
where $\mathbf h \in \mathbb R^n$ and $\mathbf g \in \mathbb R^p$ have i.i.d. standard normal entries. Next, we need to analyze the optimization~\eqref{eq:aux} to characterize its performance.

\subsection{Analyzing the auxiliary optimization}\label{sec:Analyze_AO}
In this section, we analyze the auxiliary optimization~\eqref{eq:aux}. Ideally, we would like to solve the optimizations with respect to the direction of the vectors, in order to finally get a scalar-valued optimization over the magnitude of the variables.\\
Proceeding onwards, we first perform  the maximization with respect to the direction of $\mathbf v$. We can write the following maximization with respect to $\mathbf v$,
\begin{equation}
\max_{\mathbf v \in \mathbb R^n }~ \frac{1}{n\sqrt{p}}||\mathbf v|| \mathbf g^T\mathbf P^{\perp}\bb + \frac{1}{n}\mathbf v^T\big(\mathbf u - \frac{1}{\sqrt{p}}\mathbf H_1 \bb - \frac{||\mathbf P^{\perp} \bb||}{\sqrt p} \mathbf h\big)~.
\end{equation}
In order to maximize the objective function, $\mathbf v$ chooses its direction to be the same as the vector it is multiplied to. Define $r:=||\mathbf v||/\sqrt{n}$, then maximizing over the direction of $\mathbf v$ would give,
\begin{equation}
\max_{r\geq 0}~~~r \big(\frac{1}{\sqrt{np}}\mathbf g^T\mathbf P^{\perp}\bb + ||\frac{1}{\sqrt n}\mathbf u - \frac{1}{\sqrt{np}}\mathbf H_1 \bb - \frac{||\mathbf P^{\perp} \bb||}{\sqrt{np}} \mathbf h||\big)~.
\end{equation}
Replacing this in~\eqref{eq:aux}, we would have,
\begin{align}
\label{eq:opt2}
\min_{\bb\in \mathbb R^p,\mathbf u\in \mathbb R^n}\max_{r \geq 0 }~\frac{1}{n}\mathbf 1^T\rho(\mathbf u) - \frac{1}{n}\mathbf y^T\mathbf u +\frac{\lambda}{p} f(\bb) +r\frac{1}{\sqrt{np}}\mathbf g^T\mathbf P^{\perp}\bb \nonumber \\
+ r ||\frac{1}{\sqrt n}\mathbf u - \frac{1}{\sqrt{np}}\mathbf H~(\mathbf P \bb) - \frac{||\mathbf P^{\perp} \bb||}{\sqrt{np}} \mathbf h||~,
\end{align}
where we replaced $\mathbf H_1$ with $\mathbf H\times \mathbf P$~. Next, we would like to solve the minimization with respect to $\bb$.\\
Before continuing our analysis, we need to discuss an important point that would help us in the remaining of this section. It will be observed that in order to simplify the optimization, we would like to flip the orders of $\min$ and $\max$ in the (AO) optimization.  Since the objective function in the  optimization~\eqref{eq:opt2} is not convex-concave we cannot appeal to the Sion's min-max theorem in order to flip $\min$ and $\max$. However, it has been shown in~\cite{thrampoulidis2018precise} (see Appendix~A) that flipping the order $\min$ and $\max$ in the (AO) is allowed in the asymptotic setting. This is mainly due to the fact that the original (PO) optimization was convex-concave with respect to its variables, and as the CGMT suggests (AO) and (PO) are tightly related in the asymptotic setting; hence, flipping the order of optimizations in (AO) is justified whenever such a flipping is allowed in the (PO).  We appeal to this result to flip the orders of $\min$ and $\max$ when needed.\\
The goal is to express the final result in terms of the \emph{expected Moreau envelope}  of the regularization function, $f(\cdot)$ and the link function, $\rho(\cdot)$. Finding the optimal direction of $\bb$ is cumbersome due to the existence of the term $\lambda f(\bb)$ in the objective. So, we introduce new variables $\mmu,\mathbf w \in \mathbb R^p$ and rewrite the optimization as follows,
\begin{align}
\label{eq:opt3}
\min_{\substack{\bb\in \mathbb R^p,\mathbf u\in \mathbb R^n\\\mmu  \in \mathbb R^p}}\max_{\substack{\mathbf w\in \mathbb R^p\\r \geq 0}}~\frac{1}{n}\mathbf 1^T\rho(\mathbf u) - \frac{1}{n}\mathbf y^T\mathbf u +\frac{\lambda}{p} f(\mmu) + r \frac{1}{\sqrt{np}}\mathbf g^T\mathbf P^{\perp}\bb \nonumber\\
\qquad\qquad+ r||\frac{1}{\sqrt n}\mathbf u - \frac{1}{\sqrt{np}}\mathbf H~(\mathbf P \bb) - \frac{||\mathbf P^{\perp} \bb||}{\sqrt{np}} \mathbf h||+\frac{1}{p}\mathbf w^T(\mmu-\bb)~.
\end{align}
We are now able to perform the optimization with respect to $\bb$. As explained above, we are allowed  to flip the order of $\min$ and $\max$ in the asymptotic regime. We first analyze $\min_{\bb}$ to find the optimal direction of $\bb$. To streamline the notations, we introduce the scalars $\alpha := \frac{\bb^T\bb^{*}}{||\bb^{*}||^2}$, and $\sigma:= \frac{1}{\sqrt p}||\mathbf P^{\perp}\bb||$. Also define $\mathbf q := \frac{1}{\kappa\sqrt p}\mathbf H\bb^{*}$, where $\mathbf q$ has i.i.d. standard normal entries (recall that $\mathbf H$ has i.i.d. standard normal entries). Optimizing with respect to the direction of $\mathbb P^{\perp} \bb$ yields,
\begin{align}
\label{eq:opt4}
&&\min_{\substack{\mmu\in \mathbb R^p,\mathbf u\in \mathbb R^n\\\alpha\in \mathbb R, \sigma\geq 0}}\max_{\substack{\mathbf w\in \mathbb R^p\\r \geq 0}}&~\frac{1}{n}\mathbf 1^T\rho(\mathbf u) - \frac{1}{n}\mathbf y^T\mathbf u +\lambda f(\mmu) -\sigma||\frac{1}{\sqrt p}\mathbf P^\perp (\frac{r}{\sqrt \delta}\mathbf g-\mathbf w)||\nonumber\\
&&&+r ||\frac{1}{\sqrt n}\mathbf u - \frac{\kappa\alpha}{\sqrt n}\mathbf q - \frac{\sigma}{\sqrt n} \mathbf h||+\frac{1}{p}(\mathbf P\mathbf w)^T\mmu+\frac{1}{p}(\mathbf P^\perp\mathbf w)^T\mmu-\frac{1}{p}(\mathbf P\mathbf w)^T\bb~,
\end{align}
where $\delta:=\frac{n}{p}$ is the oversampling ratio. Next, we use a trick adopted from~\cite{thrampoulidis2018precise} where by introducing two new scalar variables, namely $\upsilon$ and $\tau$, we can change $||\cdot||$ to $||\cdot||^2$ which simplifies the next steps of our analysis. The new optimization would be,
\begin{align}
\label{eq:opt5}
\min_{\substack{\mmu\in \mathbb R^p,\mathbf u\in \mathbb R^n\\\alpha\in \mathbb R, \sigma\geq 0\\\upsilon\geq 0}}~\max_{\substack{\mathbf w\in \mathbb R^p\\r,\tau \geq 0}}~\frac{1}{n}\mathbf 1^T\rho(\mathbf u) - \frac{1}{n}\mathbf y^T\mathbf u +\frac{\lambda}{p} f(\mmu) - \frac{\sigma}{2\tau}-\frac{\sigma\tau}{2}||\frac{1}{\sqrt p}\mathbf P^\perp (\frac{r}{\sqrt \delta}\mathbf g-\mathbf w)||^2 +\frac{r}{2\upsilon}\nonumber \\+ \frac{r\upsilon}{2}  ||\frac{1}{\sqrt n}\mathbf u - \frac{\kappa\alpha}{\sqrt n}\mathbf q - \frac{\sigma}{\sqrt n} \mathbf h||^2+\frac{1}{p}(\mathbf P\mathbf w)^T\mmu+\frac{1}{p}(\mathbf P^\perp\mathbf w)^T\mmu-\frac{1}{p}(\mathbf P\mathbf w)^T\bb~.
\end{align}
Next, in order to compute the optimal $\mathbf w$, we use the following completion of squares,
\begin{equation}
\label{eq:completion_of_squares_2}
\begin{aligned}
&&-\frac{\sigma\tau}{2}||\frac{1}{\sqrt p}\mathbf P^\perp (\frac{r}{\sqrt \delta}\mathbf g-\mathbf w)||^2 + \frac{1}{p}(\mathbf P^\perp \mathbf w)^T\mmu &= -\frac{\sigma\tau}{2}||\frac{1}{\sqrt p}\mathbf P^\perp (\frac{r}{\sqrt \delta}\mathbf g-\mathbf w+\frac{1}{\sigma\tau}\mmu)||^2\\
&&& + \frac{1}{2p\sigma\tau}||\mathbf P^\perp \mmu+\frac{\sigma\tau r}{\sqrt \delta}\mathbf P^\perp \mathbf g||^2 -\frac{\sigma\tau r^2}{2n}||\mathbf P^\perp \mathbf g||^2~.
\end{aligned}
\end{equation}
Since $\mathbf g\in \mathbb R^p$ has standard normal entries, we can approximate $\frac{\sigma\tau r^2}{2n}||\mathbf P^\perp \mathbf g||^2$ with $\frac{\sigma\tau r^2}{2\delta}$. We exploit~\eqref{eq:completion_of_squares_2} to solve the inner optimization with respect to $\mathbf w$ which gives,
\begin{equation}
\label{eq:opt6}
\begin{aligned}
&&\min_{\substack{\mmu\in \mathbb R^p,\mathbf u\in \mathbb R^n\\\alpha\in \mathbb R, \sigma, \upsilon\geq 0\\\frac{1}{p}{\bb^{*}}^T\mmu=\alpha\kappa^2}}~\max_{\substack{r,\tau \geq 0}}~&\frac{1}{n}\mathbf 1^T\rho(\mathbf u) - \frac{1}{n}\mathbf y^T\mathbf u - \frac{\sigma}{2\tau}-\frac{\sigma\tau r^2}{2\delta}+\frac{r}{2\upsilon} -\frac{\kappa^2\alpha^2}{2\sigma\tau} \\
&&&+ \frac{r\upsilon}{2}  ||\frac{1}{\sqrt n}\mathbf u - \frac{\kappa\alpha}{\sqrt n}\mathbf q - \frac{\sigma}{\sqrt n} \mathbf h||^2+\frac{1}{2p\sigma\tau}||\mmu+\frac{\sigma\tau r}{\sqrt \delta}\mathbf g||^2+\frac{\lambda}{p} f(\mmu)~,
\end{aligned}
\end{equation}
where we also used the following equality:
\begin{equation}
\begin{aligned}
&&\frac{1}{p}||\mathbf P^\perp \mmu+\frac{\sigma\tau r}{\sqrt \delta}\mathbf P^\perp \mathbf g||^2 &= \frac{1}{p}|| \mmu+\frac{\sigma\tau r}{\sqrt \delta} \mathbf g||^2 - \frac{1}{p}||\mathbf P\mmu||^2 -(\sigma\tau r)^2\frac{||\mathbf P\mathbf g||^2}{n}-\frac{2\sigma\tau r}{p\sqrt \delta}(\mathbf P\mathbf g)^T\mathbf \mu\\
&&\boxed{p\rightarrow +\infty} \quad\quad& = \frac{1}{p}|| \mmu+\frac{\sigma\tau r}{\sqrt \delta}\mathbf g||^2 - \frac{1}{p}||\mathbf P\mmu||^2 = \frac{1}{p}|| \mmu+\frac{\sigma\tau r}{\sqrt \delta} \mathbf g||^2 - \kappa ^2\alpha^2~.
\end{aligned}
\end{equation}
Consequently, by flipping the order of $\min$ and $\max$, we first compute the minimization with respect to $\mmu$. Hence, the optimal $\mmu$ would be the solution to the following optimization:
\begin{equation}
\label{eq:opt_mu}
\begin{aligned}
&&\min_{\mmu \in \mathbb R^p}&\quad\frac{1}{2p\sigma\tau}||\mmu-\frac{\sigma\tau r}{\sqrt \delta}\mathbf g||^2+\frac{\lambda}{p} f(\mmu)\\
&&&~\text{s.t.}~~\frac{1}{p}{\bb^{*}}^T\mmu = \alpha\kappa^2
\end{aligned}
\end{equation} 
Using the Lagrange multiplier $\theta$ we can rewrite this optimization as,
\begin{equation}
\label{eq:opt_mu2}
\begin{aligned}
&&\min_{\mmu \in \mathbb R^p}~\max_{\theta \in \mathbb R}&\quad\frac{1}{2p\sigma\tau}||\mmu-\frac{\sigma\tau r}{\sqrt \delta}\mathbf g||^2+\frac{\lambda}{p} f(\mmu)-\frac{\theta}{p}{\bb^{*}}^T\mmu +\alpha\theta\kappa^2
\end{aligned}
\end{equation}
Applying the completion of squares we have,
\begin{equation}
\label{eq:comp_square_3}
\frac{1}{2p\sigma\tau}||\mmu-\frac{\sigma\tau r}{\sqrt \delta}\mathbf g||^2 - \frac{\theta}{p}{\bb^{*}}^T\mmu = \frac{1}{2p\sigma\tau}||\mmu-\frac{\sigma\tau r}{\sqrt \delta}\mathbf g-\theta\sigma\tau\bb^{*}||^2 -\frac{\sigma\tau\theta^2\kappa^2}{2}~,
\end{equation}
where we omit the term $\frac{1}{p} {\mathbf g}^T\bb^{*} = \mathcal O(\frac{1}{\sqrt{p}})$ as its negligible compare to the other terms (which are of constant orders). We are able to represent the solution of~\eqref{eq:opt_mu} in terms of the Moreau envelope of the function $f(\cdot)$ as follows,
\begin{equation}
\label{eq:opt_mu_3}
\min_{\substack{\mmu \in \mathbb R^p\\\frac{1}{p}{\bb^{*}}^T\mmu = \alpha\kappa^2}}\frac{1}{2p\sigma\tau}||\mmu-\frac{\sigma\tau r}{\sqrt \delta}\mathbf g||^2+\frac{\lambda}{p} f(\mmu)~=~\max_{\theta \in \mathbb R}~\frac{1}{p}M_{\lambda f}\big(\sigma\tau(\frac{r}{\sqrt \delta}\mathbf g + \theta \bb^{*}), \sigma\tau\big)+\alpha\theta\kappa^2 - \frac{\sigma\tau\theta^2\kappa^2}{2}
\end{equation} 
Substituting~\eqref{eq:opt_mu_3} in~\eqref{eq:opt6}, we have the following optimization:
\begin{align}
\label{eq:opt7}
\min_{\substack{\mathbf u\in \mathbb R^n\\\alpha\in \mathbb R, \sigma, \upsilon\geq 0}}~\max_{\substack{r,\tau \geq 0\\\theta\in \mathbb R}}~\frac{1}{n}\mathbf 1^T\rho(\mathbf u) - \frac{1}{n}\mathbf y^T\mathbf u+ \frac{r\upsilon}{2}  ||\frac{1}{\sqrt n}\mathbf u - \frac{\kappa\alpha}{\sqrt n}\mathbf q - \frac{\sigma}{\sqrt n} \mathbf h||^2- \frac{\sigma}{2\tau}-\frac{\sigma\tau r^2}{2\delta}+\frac{r}{2\upsilon}  \nonumber \\
-\frac{\kappa^2\alpha^2}{2\sigma\tau}+\kappa^2\alpha\theta-\frac{\kappa^2\sigma\tau\theta^2}{2}+\frac{1}{p}M_{\lambda f(\cdot)}\big(\sigma\tau(\frac{r}{\sqrt \delta}\mathbf g + \theta \bb^{*}), \sigma\tau\big)~.
\end{align}
We now focus on the optimization with respect to $\mathbf u$. Recall that $\mathbf y = Ber\big(\rho'(\frac{1}{\sqrt p}\mathbf H\bb^{*})\big)=Ber\big(\rho'(\kappa \mathbf q)\big)$. We are interested in solving the following optimization:
\begin{equation}
\label{eq:min_u}
\min_{\mathbf u\in \mathbb R^n}~\frac{1}{n}\mathbf 1^T\rho(\mathbf u) - \frac{1}{n}\mathbf y^T\mathbf u+ \frac{r\upsilon}{2}  ||\frac{1}{\sqrt n}\mathbf u - \frac{\kappa\alpha}{\sqrt n}\mathbf q - \frac{\sigma}{\sqrt n} \mathbf h||^2~,
\end{equation}
Similar to the previous steps, we first do a completion of squares as follows,
\begin{equation}
\label{comp_square_4}
\begin{aligned}
&&- \frac{1}{n}\mathbf y^T\mathbf u+ \frac{r\upsilon}{2}  ||\frac{1}{\sqrt n}\mathbf u - \frac{\kappa\alpha}{\sqrt n}\mathbf q - \frac{\sigma}{\sqrt n} \mathbf h||^2&= \frac{r\upsilon}{2}||\frac{1}{\sqrt n}\mathbf u - \frac{\kappa\alpha}{\sqrt n}\mathbf q - \frac{\sigma}{\sqrt n} \mathbf h-\frac{1}{r\upsilon\sqrt n}\mathbf y||^2\\
&&&\quad - \frac{1}{2r\upsilon}||\mathbf y||^2 - \frac{k\alpha}{n}\mathbf y^T\mathbf q - \frac{\sigma}{n}\mathbf y^T\mathbf h~.
\end{aligned}
\end{equation}
Next, we use the distribution of $\mathbf y$ to simplify the expressions in the right-hand side of~\eqref{comp_square_4}. We can write,
\begin{equation}
\frac{1}{n}||\mathbf y||^2 = \frac{1}{n}\sum_{i=1}^n y_i^2 ~\overset{\text{WLLN}}{\underset{n\rightarrow \infty}\Longrightarrow}~\mathbb E[y_i^2]=\mathbb E[y_i]={\mathbb E}_Z[\rho^\prime(\kappa Z)] = \frac{1}{2}~,
\end{equation}
and,
\begin{equation}
\frac{1}{n}\mathbf y^T\mathbf q = \frac{1}{n}\sum_{i=1}^n y_i q_i = \frac{1}{n}\sum_{i=1}^n Ber(\rho'(\kappa q_i)) q_i~\overset{\text{WLLN}}{\underset{n\rightarrow \infty}\Longrightarrow}~\mathbb E_Z[Z\cdot \rho'(\kappa Z)] = \kappa~ \mathbb E_Z[\rho^{\prime\prime}(\kappa Z)]~,
\end{equation}
where $Z\sim\mathcal N(0,1)$. Also note that we can ignore the term $\frac{\sigma}{n}\mathbf y^T\mathbf h$ since it is of order $\frac{1}{\sqrt n}$. Hence, we are able to rewrite the optimization~\eqref{eq:min_u} with respect to $\mathbf u$ in the following form:
\begin{equation}
\label{eq:min_u_2}
\min_{\mathbf u\in \mathbb R^n}~\frac{1}{n}\mathbf 1^T\rho(\mathbf u) + \frac{r\upsilon}{2n}  ||\mathbf u - \kappa\alpha\mathbf q - \sigma \mathbf h-\frac{1}{r\upsilon}\mathbf y||^2-\frac{1}{4r\upsilon}-\kappa^2\alpha\mathbb E_Z[\rho^{\prime\prime}(\kappa Z)]~.
\end{equation}
We can rewrite the equation~\eqref{eq:min_u_2} in terms of the Moreau envelope, $M_{\rho(\cdot)}$, as follows,
\begin{equation}
\label{eq:opt8}
\begin{aligned}
&&\min_{\substack{\sigma, \upsilon\geq 0\\ \alpha\in \mathbb R}}~\max_{\substack{r,\tau \geq 0\\\theta\in \mathbb R}}~& - \frac{\sigma}{2\tau}-\frac{\sigma\tau r^2}{2\delta}+\frac{r}{2\upsilon} -\frac{\kappa^2\alpha^2}{2\sigma\tau}+\kappa^2\alpha\theta-\frac{\kappa^2\sigma\tau\theta^2}{2}-\frac{1}{4r\upsilon}-\kappa^2\alpha\mathbb E_Z[\rho^{\prime\prime}(\kappa Z)] \\
&&&+\frac{1}{p}M_{\lambda f(\cdot)}\big(\sigma\tau(\frac{r}{\sqrt \delta}\mathbf g + \theta \bb^{*}), \sigma\tau\big)+\frac{1}{n}~M_{\rho(\cdot)}\big(\kappa\alpha\mathbf q+\sigma\mathbf h+\frac{1}{r\upsilon}\mathbf y,\frac{1}{r\upsilon}\big)~.
\end{aligned}
\end{equation}
As the last step, we want to analyze the convergence properties of (AO). Recall that $f(\cdot)$ is a separable function. Therefore, using the result of Lemma~\ref{lem:separable}, we have:
\begin{equation}
    M_{\lambda f(\cdot)}\big(\sigma\tau(\frac{r}{\sqrt \delta}\mathbf g + \theta \bb^{*}), \sigma\tau\big) = \sum_{i=1}^{p}M_{\lambda \tilde f(\cdot)}\big(\sigma\tau(\frac{r}{\sqrt \delta}\mathbf g_i + \theta \bb^{*}_i), \sigma\tau\big)
\end{equation}
Using the strong law of large numbers, we have,
\begin{equation}
    \frac{1}{p}~M_{\lambda f(\cdot)}\big(\sigma\tau(\frac{r}{\sqrt \delta}\mathbf g + \theta \bb^{*}), \sigma\tau\big) \overset{a.s.}\longrightarrow \mathbb E\big[M_{\lambda \tilde f(\cdot)}\big(\sigma\tau(\frac{r}{\sqrt \delta} Z + \theta \beta), \sigma\tau\big)\big]~,
\end{equation}
where $Z$ is a standard normal random variable and $\beta\sim \Pi$ is independent of $Z$. Similarly, we can write,
\begin{equation}
    \frac{1}{n}~M_{\rho(\cdot)}\big(\kappa\alpha\mathbf q+\sigma\mathbf h+\frac{1}{r\upsilon}\mathbf y,\frac{1}{r\upsilon}\big) \overset{a.s.}\longrightarrow \mathbb E\big[M_{\rho(\cdot)}\big(\kappa \alpha Z_1 +\sigma Z_2 + \frac{1}{r\upsilon}Ber(\kappa Z_1), \frac{1}{r\upsilon} \big)\big]~.
\end{equation}
We appeal to Lemma~9 in Appendix A of~\cite{thrampoulidis2018precise} to analyze the convergence properties of (AO). Due to the convergence we are getting from the LLN, applying this lemma enables us to replace the Moreau envelopes with their expected value. Hence, We need to analyze the following optimization,
\begin{align}
\label{eq:AO_final}
\min_{\substack{\sigma, \upsilon\geq 0\\ \alpha\in \mathbb R}}~\max_{\substack{r,\tau \geq 0\\\theta\in \mathbb R}}~ - \frac{\sigma}{2\tau}-\frac{\sigma\tau r^2}{2\delta}+\frac{r}{2\upsilon} -\frac{\kappa^2\alpha^2}{2\sigma\tau}+\kappa^2\alpha\theta-\frac{\kappa^2\sigma\tau\theta^2}{2}-\frac{1}{4r\upsilon}-\kappa^2\alpha\mathbb E_Z[\rho^{\prime\prime}(\kappa Z)] \nonumber\\
+\mathbb E\big[M_{\lambda \tilde f(\cdot)}\big(\sigma\tau(\frac{r}{\sqrt \delta} Z + \theta \beta), \sigma\tau\big)\big]+\mathbb E\big[M_{\rho(\cdot)}\big(\kappa \alpha Z_1 +\sigma Z_2 + \frac{1}{r\upsilon}Ber(\kappa Z_1), \frac{1}{r\upsilon} \big)\big]~.
\end{align}
\subsection{Finding the optimality condition of the scalar optimization}
In this section, we conclude the proof of the main result of the paper. For this, we need to show that the optimizer of the optimization~\eqref{eq:AO_final} can be found by solving the nonlinear system of equations~\eqref{eq:nonlinsys}. Let $C(\alpha, \sigma. r, \tau, \upsilon, \theta) $ denote the objective function in~\eqref{eq:AO_final}. We want to find the optimer of $C(\cdot)$, i.e., the point $(\alpha^\star, \sigma^\star, r^\star, \tau^\star, \upsilon^\star, \theta^\star)$. Since the objective function is smooth, when the optimal values are all non-zero, they should satisfy the first order optimality condition, i.e., 
\begin{equation}
\label{eq:first-order}
\nabla C = \mathbf 0~.
\end{equation}
We will show that the~\eqref{eq:first-order} would simplify to our system of nonlinear equations. We start by putting the derivative w.r.t. $\theta$ equal to zero. We have the following,
\begin{equation}
\label{eq:der_theta_1}
\frac{\partial C}{\partial \theta} = 0 \Rightarrow \kappa^2\alpha - \kappa^2\sigma\tau\theta + \frac{1}{p}\mathbb  E\big[{\bb^{*}}^T\big( \tau \sigma (\frac{r}{\sqrt \delta}\mathbf g + \theta \bb^{*})-\text{Prox}_{\sigma\tau\lambda f(\cdot)}(\sigma\tau(\frac{r}{\sqrt \delta}\mathbf g + \theta \bb^{*}))\big)\big] = 0~,
\end{equation}
where we used Lemma~\ref{lem:der_moreau} for taking the derivative of the Moreau envelope,  $M_{\lambda f(\cdot)}$. We can simplify~\eqref{eq:der_theta_1} and reqrite it as follows,
\begin{equation}
    \kappa^2\alpha = \frac{1}{p}\mathbb E\big[{\bb^{*}}^T\text{Prox}_{\sigma\tau\lambda f(\cdot)}(\sigma\tau(\frac{r}{\sqrt \delta}\mathbf g + \theta \bb^{*}))\big)\big]~.
\end{equation}
Next, we take derivative of the objective  function $C(\cdot)$ w.r.t. $r$ and $\upsilon$ and put that equal to zero. We state the following lemma which will be exploited in taking the derivatives.
\begin{lem}
\label{lem:der_F}
For fixed values of $\kappa,\alpha,\text{ and }\sigma$, let the function $F:\mathbb R_{+}\rightarrow \mathbb R$ be defined as follows,
\begin{equation}
    F(\gamma)  = -\frac{1}{4}\gamma+\mathbb E_{Z_1,Z_2}\big[M_{\rho(\cdot)}\big(\kappa\alpha Z_1+\sigma Z_2+\gamma Ber(\rho'(\kappa Z_1)), \gamma\big)\big]
\end{equation},
then the derivative of $F(\cdot)$ would be as follows:
\begin{equation}
\label{eq:der_F}
    F'(\gamma) =-\frac{1}{\gamma^2}\mathbb E\big[{\rho'(-\kappa Z_1)}\big(\kappa \alpha Z_1+ \sigma Z_2-\text{Prox}_{{\gamma\rho(\cdot)}}(\kappa \alpha Z_1+ \sigma Z_2)\big)^2\big]~.
\end{equation}
\end{lem}
\begin{proof}
We have,
\begin{equation}
F'(\gamma) = -\frac{1}{4} + \frac{d}{d\gamma}~\mathbb E_{Z_1,Z_2}\big[M_{\rho(\cdot)}(\kappa\alpha^\star Z_1+\sigma^\star Z_2+\gamma Ber(\rho'(\kappa Z_1)), \gamma)\big]
\end{equation}
In order to compute the last derivative we exploit Lemma~\ref{lem:der_moreau}. We have,
\begin{equation}
\label{eq:gamma_derivative_1}
\begin{aligned}
& \frac{d}{d\gamma}~\mathbb E\big[M_{\rho(\cdot)}(\kappa\alpha Z_1+\sigma Z_2+\gamma Ber(\rho'(\kappa Z_1)), \gamma)\big] =- \mathbb  E\big[\frac{\rho'(-\kappa Z_1)}{2\gamma^2}\big(\kappa \alpha Z_1+ \sigma Z_2-\text{Prox}_{\gamma{\rho(\cdot)}}(\kappa \alpha^\star Z_1+ \sigma^\star Z_2)\big)^2\big]\\
&\qquad\qquad\qquad\qquad\qquad\qquad\qquad\qquad\quad- \mathbb E\big[\frac{\rho'(\kappa Z_1)}{2\gamma^2}\big(\kappa \alpha Z_1+ \sigma Z_2+\gamma-\text{Prox}_{{\gamma\rho(\cdot)}}(\kappa \alpha Z_1+ \sigma Z_2+\gamma)\big)^2\big]\\
&\qquad\qquad\qquad\qquad\qquad\qquad\qquad\qquad\quad + \mathbb E\big[\frac{\rho'(\kappa Z_1)}{\gamma}\big(\kappa \alpha Z_1+ \sigma Z_2+\gamma - \text{Prox}_{{\gamma\rho(\cdot)}}(\kappa \alpha Z_1+ \sigma Z_2+\gamma)\big)\big]~,
\end{aligned}
\end{equation}
where we used the fact that for $x\in \mathbb R$, $\rho'(-x) = 1-\rho'(x)  $.
To derive~\eqref{eq:der_F} we appeal to the result of Lemma~\ref{lem:prox} which gives the following identity,
\begin{equation}
    \label{eq:lem_prox_app}
    \text{Prox}_{{\gamma\rho(\cdot)}}(\kappa \alpha Z_1+ \sigma Z_2+\gamma) = -\text{Prox}_{{\gamma\rho(\cdot)}}(-\kappa \alpha Z_1- \sigma Z_2)~.
\end{equation}
\end{proof}
Next, we use the result of Lemma~\ref{lem:der_F} to find the optimality conditions with respect to $r$ and $\upsilon$. We have,
\begin{equation}
\label{eq:der_tau_r}
    \begin{cases}
    \frac{\partial}{\partial r}C = 0 \Rightarrow -\frac{\sigma\tau r}{\delta} + \frac{1}{2\upsilon} -\frac{1}{\upsilon r^2}F'(\frac{1}{\upsilon r}) + \frac{1}{p}\mathbb E\big[\frac{\mathbf g^T}{\sqrt\delta}\big(\frac{\sigma \tau r}{\sqrt \delta}\mathbf g-\text{Prox}_{\sigma\tau\lambda f(\cdot)}(\sigma\tau(\frac{r}{\sqrt \delta}\mathbf g + \theta \bb^{*}))\big)\big] = 0~,\\
    \frac{\partial}{\partial \upsilon}C = 0 \Rightarrow \frac{-r}{2\upsilon ^2} - \frac{1}{r\upsilon^2}F'(\frac{1}{r\upsilon})= 0~.
    \end{cases}
\end{equation}

In order to simplify the equations, we define a new variable $\gamma := \frac{1}{r\upsilon}$. We can rewrite the equations~\eqref{eq:der_tau_r} as follows,
\begin{equation}
    \begin{cases}
    \gamma = \frac{1}{p}\mathbb E\big[\frac{\mathbf g^T}{r \sqrt \delta}\text{Prox}_{\sigma\tau\lambda f(\cdot)}\big(\sigma\tau(\frac{r}{\sqrt \delta}\mathbf g + \theta \bb^{*})\big)\big]~,\\
    \gamma^2 = \mathbb E\big[\frac{2\rho'(-\kappa Z_1)}{r^2}\big(\kappa \alpha Z_1+ \sigma Z_2-\text{Prox}_{{\gamma\rho(\cdot)}}(\kappa \alpha Z_1+ \sigma Z_2)\big)^2\big]~.
    \end{cases}
\end{equation}
So, far we have shown that three of the optimality conditions are the same as the nonlinear equations $1$,$2$, and $5$ in~\eqref{eq:nonlinsys}. Next, we take the derivative w.r.t. $\tau$. We have,
\begin{equation}
\label{eq:der_tau}
    \frac{\partial}{\partial \tau} C = 0 \Rightarrow \frac{\sigma}{2\tau^2} - \frac{\sigma r^2}{2\delta} + \frac{\kappa^2\alpha^2}{2\sigma \tau^2}-\frac{\kappa^2\sigma\theta^2}{2} + \frac{1}{p}~\frac{\partial}{\partial \tau}\mathbb E[M_{\lambda f(\cdot)}\big(\sigma\tau(\frac{r}{\sqrt \delta}\mathbf g + \theta \bb^{*}), \sigma\tau\big)] = 0~.
\end{equation}
The derivative of the expected Moreau envelope can be computed as follows,
\begin{equation}
\label{eq:der_Mor_tau}
    \frac{1}{p}~\frac{\partial}{\partial \tau}\mathbb E[M_{\lambda f(\cdot)}\big(\sigma\tau(\frac{r}{\sqrt \delta}\mathbf g + \theta \bb^{*}), \sigma\tau\big)] = \frac{\sigma}{2}(\frac{r^2}{\delta}+\theta^2\kappa) - \frac{1}{2\sigma \tau^2}\mathbb E\big[||\text{Prox}_{\sigma\tau\lambda f(\cdot)}\big(\sigma\tau(\frac{r}{\sqrt \delta}\mathbf g + \theta \bb^{*})\big){||}_2^{2}\big]~.
\end{equation}
Replacing~\eqref{eq:der_Mor_tau} in~\eqref{eq:der_tau} would result in,
\begin{equation}
    \label{eq:der_tau_final}
    (\kappa\alpha)^2+\sigma^2 = \mathbb E\big[|\text{Prox}_{\sigma\tau\lambda f(\cdot)}\big(\sigma\tau(\frac{r}{\sqrt \delta}\mathbf g + \theta \bb^{*})\big){||}_2^{2}\big]~.
\end{equation}
which is the third equation in the nonlinear system~\eqref{eq:nonlinsys}. Next, putting the derivative w.r.t. $\sigma$ equal zero gives the following,
\begin{equation}
    \label{eq:der_sigma}
    -\frac{1}{2\tau}-\frac{\tau r^2}{2\delta}+\frac{\kappa^2\alpha^2}{2\sigma^2\tau} -\frac{\kappa^2\tau\theta^2}{2}+ \frac{1}{p}~\frac{\partial}{\partial \tau}\mathbb E[M_{\lambda f(\cdot)}\big(\sigma\tau(\frac{r}{\sqrt \delta}\mathbf g + \theta \bb^{*}), \sigma\tau\big)] + \frac{\partial}{\partial \sigma}\mathbb E\big[M_{\rho(\cdot)}\big(\kappa \alpha Z_1+\sigma Z_2 + \gamma Ber(\kappa Z_1),\gamma\big)\big]=0~.
\end{equation}
We can compute the partial derivative of the expected Moreau envelopes as follows,
\begin{equation}
\label{eq:der_Mor_sigma}
    \frac{1}{p}~\frac{\partial}{\partial \sigma}\mathbb E[M_{\lambda f(\cdot)}\big(\sigma\tau(\frac{r}{\sqrt \delta}\mathbf g + \theta \bb^{*}), \sigma\tau\big)] = \frac{\tau}{2}(\frac{r^2}{\delta}+\theta^2\kappa) - \frac{1}{2\sigma^2 \tau}\mathbb E\big[||\text{Prox}_{\sigma\tau\lambda f(\cdot)}\big(\sigma\tau(\frac{r}{\sqrt \delta}\mathbf g + \theta \bb^{*})\big){||}_2^{2}\big]~,
\end{equation}
and,
\begin{equation}
    \label{eq:der_MOr_rho_sigma}
    \begin{aligned}
    &&\frac{\partial}{\partial \sigma}\mathbb E\big[M_{\rho(\cdot)}\big(\kappa \alpha Z_1+\sigma Z_2 + \gamma Ber(\kappa Z_1),\gamma\big)\big] &= \frac{\sigma}{\gamma} - \frac{2}{\gamma}\mathbb E\big[Z_2\rho'(-\kappa Z_1)\text{Prox}_{\gamma\rho(\cdot)}\big(\kappa\alpha Z_1+\sigma Z_2\big)\big]~,\\
    &&&=\frac{\sigma}{\gamma} \big(1- 2\mathbb E\big[\frac{\rho'(-\kappa Z_1)}{1+\gamma\rho''\big(\text{Prox}_{\gamma \rho(\cdot)}(\kappa \alpha Z_1+\sigma Z_2)\big)}\big]\big)~.
    \end{aligned}
\end{equation}
To derive the last equality, we used Lemma~\ref{lem:prox_2} and Lemma~\ref{lem:Stein}. Replacing~\eqref{eq:der_Mor_sigma}, and ~\eqref{eq:der_MOr_rho_sigma} in~\eqref{eq:der_sigma} gives,
\begin{equation}
\label{eq:der_sigma_final}
    1-\frac{\gamma}{\tau\sigma} = \mathbb E\big[\frac{2\rho'(-\kappa Z_1)}{1+\gamma \rho''\big(\text{Prox}_{\gamma \rho(\cdot)}(\kappa \alpha Z_1+\sigma Z_2)\big)}\big]~.
\end{equation}
As the last step, we take the derivative with respect to $\alpha$ in order to derive the fourth equation in the nonlinear system~\eqref{eq:nonlinsys}. We have,
\begin{equation}
    \label{eq:der_alpha}
    \frac{\partial C}{\partial \alpha} = \frac{-\kappa^2 \alpha}{\sigma \tau} + \kappa^2 \theta -\kappa^2 \mathbb E[\rho''(\kappa Z)] + \frac{\partial}{\partial \alpha}\mathbb E[M_{\rho(\cdot)}\big(\kappa\alpha Z_1 +\sigma Z_2 + \gamma Ber(\rho'(\kappa Z_1)), \gamma\big)]=0~.
\end{equation}To simplify this equation we write,
\begin{equation}
\label{eq:der_Mor_alpha}
    -\kappa^2 \mathbb E[\rho''(\kappa Z)] + \frac{\partial}{\partial \alpha}\mathbb E[M_{\rho(\cdot)}\big(\kappa\alpha Z_1 +\sigma Z_2 + \gamma Ber(\rho'(\kappa Z_1)), \gamma\big)] = \frac{\kappa^2\alpha}{\gamma} - 2\mathbb E\big[\frac{\kappa}{\gamma} Z_1\rho'(-\kappa Z_1)\text{Prox}_{\gamma \rho(\cdot)}\big(\kappa \alpha Z_1+\sigma Z_2\big)\big]
\end{equation}
Replacing~\eqref{eq:der_Mor_alpha} in~\eqref{eq:der_alpha} would result,
\begin{equation}
    \label{eq:der_alpha_2}
    \frac{\gamma \kappa}{2}(\theta -\frac{\alpha}{\sigma \tau}) +\frac{\kappa \alpha}{2}= \mathbb E\big[Z_1\rho'(-\kappa Z_1)\text{Prox}_{\gamma \rho(\cdot)}\big(\kappa \alpha Z_1+\sigma Z_2\big)\big]~.
\end{equation}
Using Stein's lemma, we can rewrite the RHS as,
\begin{equation}
    \label{eq:stein_der_alpha}
    \begin{aligned}
    &&\text{RHS} &= -\mathbb E\big[\kappa \rho''(-\kappa Z_1)\text{Prox}_{\gamma \rho(\cdot)}\big(\kappa \alpha Z_1+\sigma Z_2\big)\big]  + \kappa \alpha\mathbb E[\frac{\rho'(-\kappa Z_1)}{1+\gamma \rho''\big(\text{Prox}_{\gamma \rho(\cdot)}(\kappa \alpha Z_1+\sigma Z_2)\big)}],\\
    &&& = -\mathbb E\big[\kappa \rho''(-\kappa Z_1)\text{Prox}_{\gamma \rho(\cdot)}\big(\kappa \alpha Z_1+\sigma Z_2\big)\big]  + \frac{\kappa\alpha}{2} - \frac{\kappa\alpha\gamma}{2\tau\sigma}~,
    \end{aligned}
\end{equation}
where we exploit~\eqref{eq:der_sigma_final} to derive the last equation. Substituting in~\eqref{eq:der_alpha_2} would give,
\begin{equation}
    \label{eq:der_alpha_final}
    \gamma\theta = -2\mathbb E\big[ \rho''(-\kappa Z_1)\text{Prox}_{\gamma \rho(\cdot)}\big(\kappa \alpha Z_1+\sigma Z_2\big)\big]~.
\end{equation}
Therefore, we have shown that the nonlinear system~\eqref{eq:nonlinsys} is equivalent to the optimality condition in \eqref{eq:AO_final}. 

 Recall in the process of simplifying (AO) in Section~\ref{sec:Analyze_AO}, we introduced the Moreau envelope of $f(\cdot)$ in~\eqref{eq:opt_mu_3}. The optimizer of that Moreau envelope gives the solution of the Auxiliary optimization. Let $(\bar \alpha, \bar \sigma, \bar \gamma, \bar \theta, \bar \tau, \bar r)$ be the unique solution of the nonlinear system. Hence, we can present the solution  of the (AO) in terms of the proximal operator as follows,
\begin{equation}
\label{AO_estimate}
    {\hat \bb}^{AO}_i =  \Gamma(\bb^*_i, Z)=\text{Prox}_{\lambda \bar\sigma \bar \tau\tilde f(\cdot)}\big(\bar\sigma\bar\tau(\bar \theta \bb^*_i+\frac{\bar r}{\sqrt{\delta}}Z)\big),~~~\text{for }~i=1,2,\ldots p.
\end{equation} 
As the last step we want to show the convergence of the locally-Lipschitz function $\Psi(\cdot, \cdot)$. Recall in Section~\ref{sec:find_AO}, in order to apply the CGMT, we have introduced some artificial bounded sets on the optimization variables and state that we can perform the optimization over these sets. Considering the variables belong to those bounded sets, we can state the function $\Psi(\cdot, \cdot)$ is Lipschitz, as constraining a locally-Lipschitz function to a bounded set gives a Lipschitz function. Next, using the strong law of large numbers along with the fact that the entries of $\bb^*$ are i.i.d. and drawn from distribution $\Pi$, we have,
\begin{equation}
\label{eq:conv_beta_hat}
    \frac{1}{p}\sum_{i=1}^p\Psi( {\hat \bb}^{AO}_i, \bb^*_i)\overset{a.s.}\longrightarrow \mathbb E\big[\Psi(\Gamma(\beta, Z),\beta)\big]~,
\end{equation}
where $Z$ is a standard normal random variable and $\beta\sim\Pi$ is independent of $Z$. 

Exploiting the assymptotic convergence of CGMT (Corollary~\ref{cor:CGMT}), we can introduce the set $\mathcal S$ as follows,
\begin{equation}
    \mathcal S = \{\bb\in\mathbb R^p: |\frac{1}{p}\sum_{i=1}^p\Psi( {\bb}, \bb^*_i) - \mathbb E\big[\Psi(\Gamma(\beta, Z),\beta)\big]| >\epsilon\}
\end{equation}

The convergence in~\eqref{eq:conv_beta_hat} would establish that as $p\rightarrow \infty$, ${\hat \bb}^{AO} \in \mathcal S$ with probability approaching $1$. Therefore, as the result of Corollary~\ref{cor:CGMT}, $\hat \bb  = {\hat \bb }^{PO}\in \mathcal S$ with probability approaching $1$. This concludes the proof.

\section{Proof of Theorem~\ref{thm:l2_reg}}
This result can be derived using the result of Theorem~\ref{thm:main}. We just need to show that the parameters $\theta$, $r$, and $\tau$ can be explicitely computed from the first three equations in the nonlinear system~\eqref{eq:nonlinsys}. Recall that we characterize the performance of the RLR in terms of the solution of the following nonlinear equation,
\begin{equation}
\label{eq:nonlinsys_rep}
\begin{cases}
\begin{aligned}
&&\kappa^2 \alpha &= ~\mathbb E\big[\beta~\text{Prox}_{\lambda\sigma\tau \tilde f(\cdot)}\big(\sigma\tau(\theta \beta+\frac{r}{\sqrt{\delta}}Z)\big)\big]~,\\
&&\gamma &= \frac{1}{r\sqrt{\delta}}~\mathbb E\big[Z~\text{Prox}_{\lambda\sigma\tau \tilde f(\cdot)}\big(\sigma\tau(\theta \beta+\frac{r}{\sqrt{\delta}}Z)\big)\big]~,\\
&&\kappa^2\alpha^2 + \sigma^2 &= ~ \mathbb E~\big[\text{Prox}_{\lambda\sigma\tau \tilde f(\cdot)}\big(\sigma\tau(\theta \beta+\frac{r}{\sqrt{\delta}}Z)\big)^2\big]~,\\
&&\gamma^2 &= \frac{2}{r^2}~\mathbb E\big[\rho'(-\kappa Z_1)\big(\kappa\alpha Z_1+\sigma Z_2 -\text{Prox}_{\gamma\rho(\cdot)}(\kappa\alpha Z_1+\sigma Z_2) \big)^2\big]~,\\
&&\theta\gamma&=-2~\mathbb E\big[ \rho''(-\kappa Z_1)\text{Prox}_{\gamma \rho(\cdot)}\big(\kappa \alpha Z_1+\sigma Z_2\big)\big]~,\\
&&1-\frac{\gamma}{\sigma\tau}&=~\mathbb E\big[\frac{2\rho'(-\kappa Z_1)}{1+\gamma \rho''\big(\text{Prox}_{\gamma\rho(\cdot)}(\kappa\alpha Z_1 + \sigma Z_2)\big)}\big]~.
\end{aligned}
\end{cases}
\end{equation}
In the $\ell_2^2$-regularization, we have $\tilde f(\cdot) = \frac{1}{2}(\cdot)^2$, for which the proximal operator can be computed in closed-form, i.e., we have,
\begin{equation}
    \text{Prox}_{t\tilde f}(x) = \frac{x}{1+t}~.
\end{equation}
Replacing in the first equation of~\eqref{eq:nonlinsys_rep} gives,
\begin{equation}
\label{eq:simply_1}
    \begin{aligned}
&&\kappa^2 \alpha &= ~\mathbb E\big[\beta~\text{Prox}_{\lambda\sigma\tau \tilde f(\cdot)}\big(\sigma\tau(\theta \beta+\frac{r}{\sqrt{\delta}}Z)\big)\big]\\
&&&=~\mathbb E\big[\beta \times \frac{\sigma\tau(\theta \beta+\frac{r}{\sqrt{\delta}}Z)}{1+\lambda\sigma\tau}\big] = \frac{\sigma\tau \theta \kappa^2}{1+\lambda\sigma\tau}~.
\end{aligned}
\end{equation}
where we used the fact that $\mathbb E[\beta^2] = \kappa^2$, and $\mathbb E[\beta\cdot Z]=0$. Next, from the second equation in~\eqref{eq:nonlinsys_rep} we have,
\begin{equation}
\label{eq:simply_2}
    \begin{aligned}
&&\gamma &= ~\frac{1}{r\sqrt \delta}\mathbb E\big[Z~\text{Prox}_{\lambda\sigma\tau \tilde f(\cdot)}\big(\sigma\tau(\theta \beta+\frac{r}{\sqrt{\delta}}Z)\big)\big]\\
&&&=~\frac{1}{r\sqrt \delta}\mathbb E\big[Z \times \frac{\sigma\tau(\theta \beta+\frac{r}{\sqrt{\delta}}Z)}{1+\lambda\sigma\tau}\big] = \frac{\sigma\tau}{\delta(1+\lambda\sigma\tau)}~,
\end{aligned}
\end{equation}
and finally from the thrid equation in~\eqref{eq:nonlinsys_rep} we can compute,
\begin{equation}
\label{eq:simply_3}
\begin{aligned}
    &&\kappa^2\alpha^2 +\sigma^2 &= ~\mathbb E\big[\big(~\text{Prox}_{\lambda\sigma\tau \tilde f(\cdot)}\big(\sigma\tau(\theta \beta+\frac{r}{\sqrt{\delta}}Z)\big)\big)^2\big]\\
    &&&=\frac{\sigma^2\tau^2}{(1+\lambda\sigma\tau)^2}(\theta^2\kappa^2+\frac{r^2}{\delta})\\
    &&&=\kappa^2\alpha^2 +\frac{\sigma^2\tau^2r^2}{\delta(1+\lambda\sigma\tau)^2} ~.
\end{aligned}
\end{equation}
We can rewrite the equations~\eqref{eq:simply_1},~\eqref{eq:simply_2}, and~\eqref{eq:simply_3} as follows,
\begin{equation}
\label{eq:simply_4}
    \begin{cases}
    \begin{aligned}
    &&\theta &=\frac{\alpha}{\gamma \delta}~,\\
    &&\tau &=\frac{\delta \gamma}{\sigma\big(1-\lambda \delta \gamma\big)}~,\\
    && r &= \frac{\sigma}{ \gamma \sqrt{\delta}}~.
    \end{aligned}
    \end{cases}
\end{equation}

Replacing the derived expressions in~\eqref{eq:simply_4} for $\theta, r$, and $\tau$ in the last three equations of~\eqref{eq:nonlinsys_rep} would gives the following system of three equations with three unknowns,
\begin{equation}
\label{eq:nonlin_l2_rep}
    \begin{cases}
    \begin{aligned}
    &&\frac{\sigma^2}{2\delta} &= ~\mathbb E\big[\rho'(-\kappa Z_1)\big(\kappa\alpha Z_1+\sigma Z_2 -\text{Prox}_{\gamma\rho(\cdot)}(\kappa\alpha Z_1+\sigma Z_2) \big)^2\big]~,\\
    &&-\frac{\alpha}{2\delta}&=~\mathbb E\big[ \rho''(-\kappa Z_1)\text{Prox}_{\gamma \rho(\cdot)}\big(\kappa \alpha Z_1+\sigma Z_2\big)\big]~,\\
    &&1-\frac{1}{\delta}+\lambda\gamma&=~\mathbb E\big[\frac{2\rho'(-\kappa Z_1)}{1+\gamma \rho''\big(\text{Prox}_{\gamma\rho(\cdot)}(\kappa\alpha Z_1 + \sigma Z_2)\big)}\big]~.
    \end{aligned}
    \end{cases}
\end{equation}
This concludes the proof of Theorem~\ref{thm:l2_reg}.

\end{document}